\documentclass[a4paper,11pt]{article}
\usepackage{a4wide}
\usepackage{amsmath,amsfonts,amssymb,amsthm}
\usepackage[colorlinks,citecolor=blue]{hyperref}
\usepackage{enumitem}
\usepackage{bbm}
\usepackage{natbib}
\usepackage{microtype}
\setcitestyle{square}
\usepackage{amsmath,amsfonts,amssymb,amsthm}
\usepackage[colorlinks,citecolor=blue]{hyperref}
\usepackage{enumitem}
\usepackage{bbm}
\usepackage{microtype}
\usepackage{bm}
\usepackage{graphicx}
\usepackage{subfigure}
\usepackage{tikz}
\usetikzlibrary{arrows,arrows.meta}
\newtheorem{theorem}{Theorem}[section]
\newtheorem{lemma}[theorem]{Lemma}

\theoremstyle{definition}
\newtheorem{definition}[theorem]{Definition}

\newtheorem{remark}[theorem]{Remark}

\numberwithin{equation}{section}
\numberwithin{table}{section}

\begin{document}

\title{Higher Order Approximation Rates for ReLU CNNs in Korobov Spaces}

\author{
Yuwen Li\thanks{School of Mathematical Sciences, Zhejiang University, Hangzhou, Zhejiang 310058, China. (E-mail: liyuwen@zju.edu.cn)}
\and
Guozhi Zhang\thanks{School of Mathematical Sciences, Zhejiang University, Hangzhou, Zhejiang 310058, China. (Corresponding author, E-mail: gzzh@zju.edu.cn)}
}
\date{}
\maketitle

\begin{abstract}
This paper investigates the $L_p$ approximation error for higher order Korobov functions using deep convolutional neural networks (CNNs) with rectified linear unit (ReLU) activation functions. For target functions having a mixed derivative of order $m+1$ in each direction, we improve  classical approximation rate of second order to $(m+1)$-th order (modulo a logarithmic factor) in terms of the depth of CNNs. The key ingredients in our analysis include the $L_p$ approximation error of sparse grid functions within higher order Korobov spaces, the precise ReLU product factors decomposition of high-order sparse grid basis functions, and the approximation of polynomials with non-negative bounded input variables by CNNs. The results suggest that higher order expressivity of CNNs does not severely suffer from the curse of dimensionality.


%\noindent \textbf{Mathematics Subject Classification:} 
\end{abstract}

\section{Introduction}
Deep Neural Networks (DNNs)  have become increasingly popular in many domains of scientific and technology, significant successes have been achieved in recent years in areas such as protein structure prediction \cite{JumperEvansPritzelGreenFigurnov2021}, weather forecasting \cite{BiXieZhangChenGuTian2023}, %accelerating sparse matrix multiplication \cite{ZhangAttaluriEmerSanchez2021}, 
solving partial differential equations (PDEs) \cite{De_Jagtap_Mishra2024error,RaissiPerdikarisKarniadakis2019,LiHuangLiuAnandkumar2023,LuJinPangZhangKarniadakis2021learning,LiKovachkiAzizzadenesheli2021fourier,HeLiuXu2024mgno}.
Recent theoretical research has established approximation rates of DNNs for a wide range of function classes, e.g.,
Lipschitz functions \cite{Yarotsky2017,Shen_Yang_Zhang2019,Shen_Yang_Zhang2020,Shen_Yang_Zhang2022optimal}, piecewise smooth
functions \cite{Petersen_Voigtlaender2018}, Sobolev functions \cite{Siegel2023JMLR,Suzuki2018adaptivity,Yang2024} etc.

Convolutional
Neural Networks (CNNs) are special DNNs using discrete convolutions at some layers. CNNs have achieved great success in applications including image classification, image/speech recognition \cite{Dosovitskiy2020,Krizhevsky_Sutskever_Hinton2012,Simonyan_Zisserman2014} and theoretical understanding of their approximation power is currently under intensive investigation, see, e.g., \cite{liLiuYangPengZhou2021survey,BaoLiShenTaiWuXiang2014,HeLiXu2022,PetersenVoigtlaender2020,Fang_Feng_Huang_Zhou2020,Shen_Jiao_Lin_Huang2022approximation,Mao_Zhou2023,Zhou2020a,Zhou2020b,li2024approximation}. For example, the universal approximation property of deep CNNs with one-dimensional
input has been shown in  \cite{Zhou2020b}. Moreover, approximation rates of CNNs in Sobolev and H\"{o}lder norms have been established in e.g., \cite{Zhou2020b,Mao_Zhou2023,YangZhou2025shallowreluk}. 
Universal approximation property of multi-channel CNNs with image input could be found in e.g., \cite{HeLiXu2022}. 

Besides Lipschitz/Sobolev/H\"{o}lder target functions, neural network approximations of Korobov functions have recently been studied in e.g., \cite{Blanchard_Bennouna2022,Hadrien_Du2019,Mao_Zhou2022,Yang_Lu2024,Liu_Mao_Zhou2024,Liu_Liu_Zhou_Zhou2025}. In particular, the analysis in \cite{Blanchard_Bennouna2022,Hadrien_Du2019,Mao_Zhou2022,Yang_Lu2024} relies on mechanism from sparse grids (cf.~\cite{Bungartz_Griebel2004}) for maintaining satisfactory expressivity of deep fully connected neural networks and deep CNNs as dimension grows. As far as we know, existing results in this direction are basically approximation rates $O(L^{-2+\frac{1}{p}})$ (in terms of the depth $L$ of neural networks) in $L_p(\Omega)$ based on second order mixed derivative, i.e., Korobov regularity of the target function. In this work, we shall develop arbitrarily higher order approximation rates of deep CNNs (see Theorem \ref{thm:mainresult}) by exploiting higher order Korobov regularity of target functions, generalizing the low order approximation results of deep CNNs in \cite{Mao_Zhou2022}.

The CNNs considered in this paper make use of the ReLU activation function $
\sigma(x)=\max(x,0)$ and a discrete convolution  depending on a filter vector $\bm w=(w_0,w_1,\ldots,w_s)^\top\in \mathbb{R}^{s+1}$ for some filter length $s\geq1$. Given an input vector $\bm{y}=(y_1,\ldots,y_n)^\top\in \mathbb{R}^n$, discrete convolution of the filter $\bm w$ with $\bm y$ is a prolonged  vector $\bm w *\bm  y \in \mathbb R^{n+s}$, where its $i$-th entry is  
\begin{equation*}
(\bm w * \bm y)_i=
        \sum_{k=1}^n w_{i-k} y_k,\quad 1\leq i\leq n+s.
\end{equation*} 
Clearly $\bm w * \bm y$ can be written as a matrix-vector multiplication
$
\bm w *\bm  y = T_{\bm w}  \bm  y,
$
where $T_{\bm w} \in \mathbb{R}^{(n+s) \times n}$ is a Toeplitz matrix:
\begin{equation}\label{Toeplitzmatrix}
\begin{aligned}
T_{\bm w}:=\left[\begin{array}{ccccccc}
w_0 & 0 & 0 & 0 & \ldots & 0 & 0 \\
w_1 & w_0 & 0 & 0 & \ldots & 0 & 0 \\
\vdots & \vdots & \ddots & \ddots & \ddots & \vdots & \vdots \\
w_s & w_{s-1} & \ldots & w_0 & \ldots & 0 & 0 \\
0 & w_s & \ldots & w_1 & \ddots & \vdots & 0 \\
\vdots & \ddots & \ddots & \ddots & \ddots & \ddots & \vdots \\
\ldots & \ldots & 0 & w_s & \ldots & w_1 & w_0 \\
\ldots & \ldots & \ldots & 0 & w_s & \ldots & w_1 \\
\vdots & \ldots & \ldots & \ddots & \ddots & \ddots & \vdots \\
0 & \ldots & \ldots & \ldots & \ldots & 0 & w_s
\end{array}\right]. 
\end{aligned}
\end{equation}
%It is precisely this form of sparse matrix that results in significant differences between the induced CNNs and classical neural networks. 
Each layer of a CNN consists of a nonlinear activation and a sparse affine transformation $A_{\bm{w},\bm{b}}: \mathbb{R}^n\rightarrow\mathbb{R}^{n+s}$ given by 
\begin{equation*}
A_{\bm{w},\bm{b}}(\bm{x}):=\bm{w}*\bm{x}+\bm{b},
\end{equation*} 
where $\bm w \in \mathbb R^{s+1}$ is a filter vector with filter length $s$, $\bm x \in \mathbb R^n$ and $\bm b \in \mathbb R^{n+s}$.
As the depth increases, the width of a CNN also increases by $s$ in each layer. 
\begin{definition}\label{def:DCNN}
Let $s\geq2$ be a fixed integer and $L \in \mathbb{N}_+$. By $\mathcal{H}^{s,d}_L$ we denote the set of CNN functions $f_{L}=\bm{c}\cdot h_{L}: \mathbb{R}^d\rightarrow\mathbb{R}$, where $\bm{c}\in\mathbb{R}^{d+Ls}$, and 
\begin{equation}\label{def:hL}
    h_L=\sigma\circ A_{\bm{w}_L,\bm{b}_L}\circ\sigma\circ\cdots\circ A_{\bm{w}_2,\bm{b}_2}\circ\sigma\circ A_{\bm{w}_1,\bm{b}_1},
\end{equation}
with each $\bm w_{\ell}\in \mathbb R^{s+1}$,  $\bm b_{\ell}\in \mathbb R^{d_{\ell}}$, $d_{\ell} = d+\ell s$.
\end{definition}

The activation function $\sigma$ in \eqref{def:hL}
acts on vectors in a componentwise way. In Definition \ref{def:DCNN}, the number $L$ is called the depth of the CNN.
Let $\mathbb{N}_+$ be the set of positive integers and $\mathbb{N}=\mathbb{N}_+\cup\{0\}$. For a multi-index $\bm{\alpha}=(\alpha_1,\ldots,\alpha_d) \in \mathbb{N}^d$, we define $ |\bm{\alpha}|_{\infty}:=\max _{1 \leq j \leq d} \alpha_j,\,|\bm{\alpha}|_1:=\sum_{j=1}^d \alpha_j
$ and the mixed derivative 
\begin{equation*}
    D^{\bm \alpha} f:=\frac{\partial^{|\bm \alpha|_1} f}{\partial x_1^{\alpha_1} \cdots \partial x_d^{\alpha_d}}.
\end{equation*}
The target function under consideration is in the following Korobov space.
\begin{definition}[Korobov space]
Let $\Omega=[0,1]^d$ be the unit cube in $\mathbb{R}^d$. The Korobov space corresponding to $k\in\mathbb{N}_+$, $p\in[1,\infty]$ on $\Omega$ is  
$$
K^k_p(\Omega):=\left\{f\in L_p(\Omega): D^{\bm\alpha}f \in L_p(\Omega)\text{ for each }\bm \alpha\in\mathbb{N}^d\text{ with }|\bm \alpha|_{\infty} \leq k,~f|_{\partial \Omega}=0\right\}.
$$
\end{definition}

Our main result is the following theorem which will be proved in Section \ref{sec:proof}. It shows that the CNNs can approximate functions in Korobov space $K_p^{m+1}(\Omega)$ with $m\geq 2$ with a higher order approximation rate as $m$ increases.  In addition, for $m=1$ and $p=\infty$, the result regarding the depth and approximation rate concerning $N$ is the same as \cite{Mao_Zhou2022}, dispite the pre-factor about $d$ being replaced by $d^4$ instead of $ d^2 \log_2 d$.

\begin{theorem}\label{thm:mainresult}
Let $\Omega = [0,1]^d$ be the unit cube in $\mathbb R^d$ and let $1\leq p \leq \infty$. Then for sufficiently large $N$ there exists an $L\leq C_s d^4 m^3 N (\log_2 N)$ such that 
$$
\inf_{f_L \in \mathcal{H}^{s,d}_L} \|f-f_L\|_{L_p(\Omega)} \leq 
C_{m,d}\|D^{\bm{m}+\bm{1}} f\|_{L_p(\Omega)} N^{-m-1}\left(\log _2 N\right)^{(m+2)(d-1)},
$$
where $\bm{m}+\bm{1}=(m+1,m+1,\ldots,m+1)\in\mathbb{N}_+^d$.
\end{theorem}
As shown in \cite{Lu_Shen_Yang_Zhang2021,Siegel2023JMLR}, Sobolev functions in $W_p^m(\Omega)$ could be approximated by DNNs with depth $L$ with accuracy $O(L^{-2m/d})$ in the $L_p$ norm. Despite exponential dependence of $C_{m,d}$ on $m, d$, Theorem \ref{thm:mainresult} suggests that the dimensional influence on CNN approximation error for Korobov functions is not as substantial as the one $O(L^{-2m/d})$ for Sobolev functions. Therefore, the result in Theorem \ref{thm:mainresult} mitigates the \emph{curse of dimensionality}.

\begin{remark} To achieve higher-order approximations for smoother functions, a common approach is to employ smoother activation functions, such as $\text{ReLU}^k$ ($k\ge2$), in shallow neural networks ($\text{ReLU}^k$ SNNs), which are characterized by a single hidden layer (cf. \cite{bach2017breaking,YangZhou2025shallowreluk,HeMaoXu2023expressivity,LiuMaoXu2025integral,SiegelXuFoCM}).  However, despite the superior approximation rates achieved by $\text{ReLU}^k$ SNNs for specific function classes, such as H\"older functions \cite{YangZhou2025shallowreluk} and Sobolev functions \cite{MaoSiegelXu2026approximation}, a saturation phenomenon still persists regarding the approximation order (cf. \cite{YangZhou2025shallowreluk, LiuMaoXu2025integral,MaoXu2025sharp}). Specifically, when using $\text{ReLU}^k$ SNNs with $n$ neurons to approximate  functions on $[0,1]^d$ with smoothness order $\alpha<(d+2k+1)/2$, the nearly optimal approximation error of order $O(n^{-\alpha/d})$ can be obtained. Furthermore, investigating the interplay between the regularity of Korobov functions and the approximation power of $\text{ReLU}^k$ SNNs would be of significant interest. We conjecture that for Korobov functions of order $\alpha < (d + 2k + 1)/2$, a nearly optimal approximation error of order $O(n^{-\alpha})$ can be achieved using $\text{ReLU}^k$ SNNs with $n$ neurons. This will be investigated in future work.
\end{remark}

The rest of the paper is organized as follows. In Section \ref{sec2}, we present some preliminaries and notation about sparse grids and CNNs. In Section \ref{sec:polyapproximation}, we analyze approximation accuracy of CNNs for polynomials. Section \ref{sec:proof} is devoted to the proof of approximation error bounds of CNNs for Korobov functions. Conclusions and possible future research directions could be found in Section \ref{sec:conclusion}.

\section{Preliminaries and notation}\label{sec2}
In this section, we collect results in \cite{Bungartz_Griebel2004,Mao_Zhou2022} that will be used later
on. Let us first summarize basic notation as follows.
\begin{itemize}

\item A vector or a multi-index is always denoted as some boldface letter such as $\bm y$, where $y_i$ is the $i$-th entry of $\bm y$.

\item For  a concrete integer such as 2, the boldface version $\bm{2}=(2,\ldots,2)$ is a multi-index with identical entries. 
\item For $r\in\mathbb{R}$, $q\in\mathbb{N}_+$, $\bm{r}_q:=(r,\ldots,r)\in \mathbb R^q$ is a vector with $q$ identical entries.
\item Given $\bm{x}, \bm{y}\in\mathbb{R}^q$, by $\bm{x}\odot\bm{y}$ we denote another vector in $\mathbb{R}^q$ with  $(\bm{x}\odot\bm{y})_i=x_iy_i$; by $\bm{x}\leq\bm{y}$ we denote $x_i\leq y_i$ for each $i$.

\item The input and output vectors of a  CNN are always column vectors.

\item Columns vectors $\bm{y}_1\in\mathbb{R}^{n_1}, \ldots, \bm{y}_k\in\mathbb{R}^{n_k}$ could be concatenated as a higher dimensional vector $[\bm{y}_1;\ldots;\bm{y}_k]\in\mathbb{R}^{n_1+\cdots+n_k}$. 

\item For $p\in[1,\infty]$, let $p^\prime:=\frac{p}{p-1}$ denote the conjugate index of $p$.

\item 
The number of elements in a set $\mathcal{S}$ is denoted by $|\mathcal{S}|$.

\item For any $x \in \mathbb R$, let $\lceil x\rceil := \min \{ n : n > 
x, n \in \mathbb{Z}\}$.

\item In our analysis, $C_{m,d,s,\ldots}$ is a generic constant that may change from line to line and is dependent only on $m, d, s, \ldots$

\end{itemize}

\subsection{Interpolation on Sparse Grids}
Sparse grids is an efficient numerical method for solving high-dimensional PDEs (cf.~\cite{Bungartz1998,Bungartz_Griebel2004,AlpertBeylkinGinesVozovoi2002,GradinaruHiptmair2003,WangTangGuoCheng2016}). To prove main results, we shall make use of approximation power of sparse grid  functions in high-dimensional space.  For any $\bm l\in \mathbb N_+^d$, we consider a $d$-dimensional rectangular grid $\mathcal{T}_{\bm l}$ in $\Omega=[0,1]^d$ with mesh size $\bm{h}_{\bm{l}}=\left(h_{l_1}, \ldots, h_{l_d}\right)$, i.e., the mesh size in the $i$-th coordinate direction is  $h_{l_i}=2^{-l_i}$. For each $\bm{0}\leq \bm{i}\leq 2^{\bm{l}}=(2^{l_1}
,\ldots, 2^{l_d})$, the $\bm{i}$-th grid point in $\mathcal{T}_{\bm l}$ is 
\begin{equation*}
    \bm{x}_{\bm{l}, \bm{i}}=\left(x_{l_1, i_1}, \ldots, x_{l_d, i_d}\right):=(i_1h_{l_1},\ldots,i_dh_{l_d}).
\end{equation*}
Let $
\phi(x):= \sigma(1-|x|)
$ and 
\begin{align*}
\phi_{l_j, i_j}\left(x_j\right)&:=\phi\left(\frac{x_j-i_j \cdot h_{l_j}}{h_{l_j}}\right), \\
\phi^1_{\bm{l}, \bm{i}}(\bm{x})&:=\prod_{j=1}^d \phi_{l_j, i_j}\left(x_j\right).
\end{align*} 
The work \cite{Mao_Zhou2022} developed approximation rates of CNNs for a target function $f\in K_p^2(\Omega)$ by utilizing the piecewise linear sparse grid interpolation (cf.~\cite{Bungartz_Griebel2004})
\begin{equation}\label{CPWLinterpolation}
    I_n^1f(\bm x)=\sum_{|\bm{l}|_1\leq n+d-1} \sum_{\bm{i} \in\mathcal{I}_{\bm l}} v^1_{\bm{l}, \bm{i}} \phi^1_{\bm{l}, \bm{i}}(\bm x),
\end{equation}
where $\mathcal{I}_{\bm l}$ is a hierarchical index set given by 
\begin{equation}\label{Eq:i_l}
\mathcal{I}_{\bm l}:=\left\{\bm{i} \in \mathbb{N}_+^d: \bm{1} \leq \bm{i} \leq 2^{\bm{l}}-\mathbf{1},~i_j \text { is odd for all } 1 \leq j \leq d\right\},
\end{equation}
and the coefficient $v_{\bm{l},\bm{i}}^1$ is
$$
v_{\bm{l},\bm{i}}^1=(-1)^d2^{-|\bm l|_1-d}\int_{\Omega}\phi^1_{\bm l,\bm i}(\bm x)D^{\bm2}f(\bm x) d \bm x.
$$

For $ f \in K^{m+1}_p(\Omega) $ with $m \geq 2$, we shall exploit a higher order analogue of \eqref{CPWLinterpolation}, e.g., the hierarchical
Lagrangian interpolation used in \cite{Bungartz1997,Bungartz1998}. Similarly to the piecewise linear case, we need a univariate polynomial $\phi_{l_j, i_j}^{\alpha_j}\left(x_j\right)$ of degree $\alpha_j\leq m$ with ${\rm supp}(\phi_{l_j, i_j}^{\alpha_j})=[x_{l_j,i_j}-h_{l_j},x_{l_j,i_j}+h_{l_j}]$ and construct the  tensor product basis
\begin{equation}\label{philialpha}
    \phi_{\bm{l}, \bm{i}}^{\bm \alpha}(\bm{x}):=\prod_{j=1}^d \phi^{\alpha_j}_{l_j, i_j}\left(x_j\right).
\end{equation}
The function $\phi_{l_j, i_j}^{\alpha_j}$ of degree $\alpha_j$ vanishes outside $\left[ x_{l_j, i_j}-h_{l_j}, x_{l_j, i_j}+h_{l_j}\right]$ and is uniquely defined on $\left[ x_{ l_j, i_j}-h_{l_j}, x_{l_j, i_j}+h_{l_j}\right]$ by the following $\alpha_j+1$ conditions:
\begin{equation}\label{interpolation_condition}
\phi_{l_j, i_j}^{\alpha_j}\left(x_{l_j, i_j}\right)=1, \quad \phi_{l_j, i_j}^{\alpha_j}\left(x_{l_j,i_j,k}\right)=0,\quad k=1,2,\ldots,\alpha_j,
\end{equation}
where $x_{l_j,i_j,1}=x_{l_j, i_j}+h_{l_j}$, $x_{l_j,i_j,2}=x_{l_j, i_j}-h_{l_j}$, and $x_{l_j,i_j,3}, \ldots, x_{l_j,i_j,\alpha_j}$ are the $\alpha_j-2$ hierarchical ancestors of $x_{l_j, i_j}$ ($x_{l_j,i_j,0}:=x_{l_j,i_j}$ by convention). 
At each level $l$, the grid point $x_{l,i}$ has $l+1$ ancestors (see Figure \ref{ancestor}) and two of them are $x_{l, i}\pm h_{l}$, which implies $l_j-1 \geq\alpha_j-2$. As a result, for $\bm{l} \in \mathbb N_+^d$, the degree
$\bm{\alpha}$ is set as  
\begin{equation}\label{eqn:alpha}
\mathbf{2} \leq\bm{\alpha}=\bm{\alpha}(m,\bm{l}):= \min \{m \cdot \bm{1}, \bm{l} + \bm{1}\}.
\end{equation}

\begin{figure}[htbp]
\centering 
\includegraphics[scale=0.6]{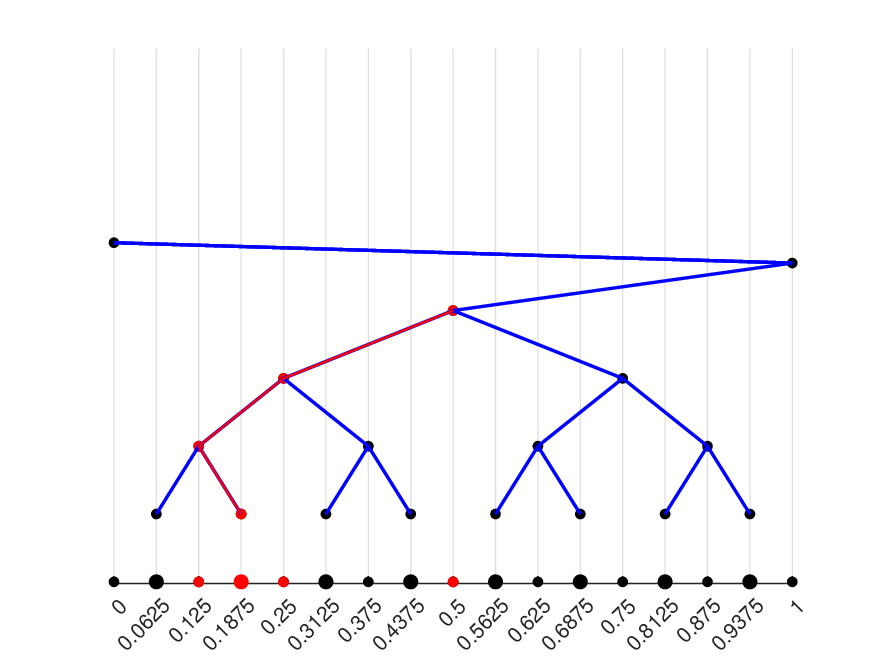}
\caption{Hierarchical ancestors of $x_{l,i} = 0.1875$ ($d=1$, $l=4$, $i=3$, $\alpha=3$) including its two neighbor points $x_{3,1}=0.125$, $x_{2,1}=0.25$ and another ancestor $x_{1,1}=0.5$.}\label{ancestor}
\end{figure}

When 
$m=2$, we have $\bm{\alpha}=\bm{2}=(2,\ldots,2)$ and obtain the basis in \eqref{philialpha} with 
\begin{equation}\label{eqn:phi2}
\phi_{l_j, i_j}^2(x)=\begin{cases}
1-\left(\frac{x-x_{l_j,i_j}}{h_{l_j}}\right)^2,\,&x \in [x_{l_j,i_j}-h_{l_j},x_{l_j,i_j}+h_{l_j}],\\
0,\,&x \notin [x_{l_j,i_j}-h_{l_j},x_{l_j,i_j}+h_{l_j}].
\end{cases}
\end{equation}
For $m \geq 3$, the univariate basis polynomial in \eqref{philialpha} is 
\begin{equation}\label{eqn:phialpha_j}
\phi_{l_j, i_j}^{\alpha_j}\left(x\right):=\begin{cases}
\prod_{k=1}^{\alpha_j}\frac{x-x_{l_j,i_j,k}}{x_{l_j,i_j}-x_{l_j,i_j,k}},& x \in [x_{l_j,i_j}-h_{l_j},x_{l_j,i_j}+h_{l_j}],\\
  0,&x \notin [x_{l_j,i_j}-h_{l_j},x_{l_j,i_j}+h_{l_j}].
\end{cases}
\end{equation} 
For a grid point $x_{l_j,i_j}$ and its $\alpha_j+1$ direct hierarchical
ancestors, denoted as $x_{l_j,i_j,0},...,x_{l_j,i_j,\alpha_{j+1}}$ in increasing order, we take 
$
w_{l_j, i_j}(t):=\prod_{k=0}^{\alpha_j+1}\left(t-x_{l_j,i_j,k}\right)$ for $t\in \mathbb R$ and the
minimum support spline with the above $\alpha_j+2$ points (cf. \cite[(4.7)]{Bungartz_Griebel2004}):$$
s_{l_j, i_j}^{\alpha_j}(t):=\sum_{k=0}^{\alpha_j+1} \frac{[\sigma(x_{l_j,i_j,k}-t)]^{\alpha_j}}{w_{l_j, i_j}^{\prime}\left(x_{l_j,i_j,k}\right)},\quad t\in \mathbb R,
$$
where $w_{l_j, i_j}^{\prime}$ is the derivative of the function $w_{l_j, i_j}$. In Section \ref{sec:proof}, we shall derive CNN approximation rates for $f\in K_p^{m+1}(\Omega)$ from the higher order sparse grid interpolation  
\begin{equation}\label{interpolation_m}
I_nf(\bm{x})=\sum_{|\bm l|_1 \leq n+d-1} \sum_{\bm i \in \mathcal{I}_{\bm l}} v_{\bm l, \bm i}\phi_{\bm l, \bm i}^{\bm{\alpha}}(\bm {x})
\end{equation}
where each $v_{\bm{l},\bm{i}}$ is given in the next lemma.
\begin{lemma}\label{lemma:vlialpha}\cite[Lemma 4.5]{Bungartz_Griebel2004}
For $f \in K_p^{m+1}(\Omega)$, the coefficient $v_{ \bm l, \bm i}$ in \eqref{interpolation_m} is
\begin{equation*}
v_{ \bm l,  \bm i}=\int_{\Omega} g_{ \bm l,  \bm i}^{\bm{\alpha}}(\bm x) \cdot D^{\bm{\alpha}+\bm{1}} f(\bm{x}) \mathrm{d} \bm{x},
\end{equation*}   
where $g_{\bm{l}, \bm{i}}^{\bm{\alpha}}(\bm{x}):=\prod_{j=1}^dw_{l_j, i_j}^{\prime}\left(x_{l_j, i_j}\right) s_{l_j, i_j}^{\alpha_j}\left(x_j\right)/\alpha_j!$.
\end{lemma}

The approximate error $\|f-I_nf\|_{L_p(\Omega)}$ follows from a series of lemmata below.
\begin{lemma}\label{lemma:philialpha_Lpbound}\cite[Lemma 4.4]{Bungartz_Griebel2004}
For any $\phi_{\bm l, \bm {i}}^{\bm{\alpha}}(\bm x)$  and $1\leq p\leq\infty$, it holds that 
\begin{equation*}
\left\|\phi_{\bm l, \bm i}^{\bm{\alpha}}\right\|_{L_p(\Omega)}  \leq 1.117^d \cdot 2^{d / p} \cdot 2^{-|\bm l|_1 / p}.
\end{equation*}
\end{lemma}

\begin{lemma}\label{Le:coeffient_p_bound}
Let $f \in K_p^{m+1}(\Omega)$ with $1 \leq p \leq \infty$, and
$$
c(\bm{\alpha}):=\prod_{j=1}^d \frac{2^{{\alpha}_j\left({\alpha}_j+1\right) / 2}}{\left({\alpha}_j+1\right)!}.
$$
Then it holds that 
\begin{equation*}
|v_{ \bm l, \bm i}|\leq c(\bm{\alpha}) \cdot 2^{-d-|\bm{l} \odot\bm{\alpha}|_1-\frac{|\bm l|_1}{p^\prime}}\cdot\left\|D^{\bm{\alpha}+\bm{1}} f\right\|_{L_p(\operatorname{supp}(\phi_{\bm{l}, \bm{i}}^{\bm \alpha}))}.
\end{equation*}
\end{lemma}

Lemma \ref{Le:coeffient_p_bound} with $p=2, \infty$ has been proved in \cite{Bungartz_Griebel2004} and case $p\in[1,\infty]$ will be proved in the appendix. Using the above lemma, we shall also prove the following result in the appendix, while the case $p=2, \infty$ has been proved in \cite{Bungartz_Griebel2004}.
\begin{lemma}\label{lemma:interpolationerror}
For any $
f \in K_p^{m+1}(\Omega)$ with $1 \leq p \leq \infty$ and sufficiently large $n$, the sparse grid interpolation $I_nf$ in \eqref{interpolation_m} satisfies
$$
\|f-I_nf\|_{L_p(\Omega)} \leq C_{m,d} \|D^{\bm{m}+\bm{1}} f\|_{L_p(\Omega)} N^{-m-1} \left(\log _2 N\right)^{ (m+2)(d-1)},
$$
where $N=|\left\{(\bm l,\bm i):|\bm{l}| \leq n+d-1,\bm{i} \in \mathcal{I}_{\bm l}\right\}|=O(2^n n^{d-1})$.
\end{lemma}

\subsection{Representing Shallow Networks by Deep CNNs }
In the end of this section, we present a key lemma developed in \cite{Zhou2020b,Mao_Zhou2022} relating any deep CNN with a shallow one, which is crucial for the subsequent analysis. Before presenting the lemma, we clarify some ambiguity in the notation of CNNs. For two vectors $\bm w$, $\bm{a}$, let $\tilde{\bm{a}}=[\bm 0_{n_1};\bm{a};\bm 0_{n_2}].$
Straightforward calculation shows that 
\begin{equation*}
     \bm w*\tilde{\bm{a}}= [\bm 0_{n_1};\bm w*\bm a ;\bm 0_{n_2}].
\end{equation*}
As a result, the output of a neural network remains the same no matter whether the input  is $ \bm a$ or $[\bm 0_{n_1};\bm a;\bm 0_{n_2}]$. In other words, zero elements at both ends of an input vector for a CNN could be ignored without affecting the output of that CNN. By
$$
\bm x \stackrel{h_{L}}\longrightarrow \bm y
$$
we denote $h_L(\bm{x})=[\bm 0_{n_1}; \bm y; \bm 0_{n_2}]$ for some $n_1, n_2\in\mathbb{N}$. 
\begin{lemma}\label{Le:Represent_shallow_net}\cite[Lemma 1]{Mao_Zhou2022}
Let $n_0, n \in \mathbb{N}_+$ and $M>0$ be a constant. For any $\bm{w}\in\mathbb{R}^{n+1}$, $\bm{b}\in \mathbb{R}^{n_0+n}$, there exists  $h_L$ of the form \eqref{def:hL} with $L \leq\left\lceil\frac{n}{s-1}\right\rceil$ such that
\begin{equation*}
    \bm{x} \stackrel{h_{L}}\longrightarrow\sigma\left(\bm{w}*{\bm x}+{\bm b}\right),\quad\forall\bm{x}\in[0,M]^{n_0}.
\end{equation*}
\end{lemma}
%    The above notation denotes from the input vector $\bm x$, the output vector of the $L$-th layer of the CNN  is $\sigma\left(T^{\bm w}{\bm x}+{\bm b}\right)$. In the rest of the paper, for convenience, we will use this notation consistently. 

\section{Approximating Polynomials by CNNs}\label{sec:polyapproximation}
Recall that a sparse grid interpolant $I_nf$ is a linear combination of products, i.e., a polynomial of one dimensional functions $\phi_{l_j,i_j}^{\alpha_j}$. In order to use approximation results of sparse grids, it is natural to approximate multivariate polynomials by CNNs and represent each $\phi_{l_j,i_j}^{\alpha_j}$ by piecewise linear functions.  The former goal is further reduced to CNN approximation of the binary product $(x,y)\mapsto xy$ that will be addressed in this section, while the latter will be left in Section \ref{sec:proof}.  

%Similar to the case of ReLU fully connected networks approximating the product function $xy$, which has been discussed in various literature such as \cite{Yarotsky2017,Yang2024}, we will construct the CNNs to approximate $xy$ in this section.

Neural network approximation $\widetilde{\times}_{M,U}(x,y)$ for $(x,y)\mapsto xy$ has been extensively used in approximation theory of DNNs and CNNs, see, e.g., \cite{Yarotsky2017,Lu_Shen_Yang_Zhang2021,Mao_Zhou2022}. As usual, we begin with sawtooth functions $T_i: [0,1] \rightarrow[0,1]$, where $T_i=T_1 \circ T_{i-1} $,
$$
T_1(x):= \begin{cases}2 x, & x \in[0,1 / 2], \\ 2(1-x), & x \in(1 / 2,1]. \\ \end{cases}
$$
see Figure \ref{Fig.2a}. For $U\in\mathbb{N}_+$, it is known that 
$$
R_U(x):=x-\sum_{i=1}^U \frac{T_i(x)}{4^i}
$$
(cf.~\cite[Lemma 5.1]{Lu_Shen_Yang_Zhang2021}, \cite[Appendix A.2]{Elbrachter_Grohs_Jentzen_Schwab2022}) is the piecewise linear interpolant of $x^2$ at grid points $\frac{i}{2^U}$ with $0\leq i\leq 2^U$, see Figure \ref{Fig.2b}. In addition, it holds that 
\begin{equation}\label{eqn:R_Uerror}
0 \leq R_U(x)-x^2 \leq 2^{-2 U-2},\quad x \in[0,1].
\end{equation}
\begin{figure}[htbp]
      \centering
      \subfigure[]{
           \includegraphics[scale=0.45]{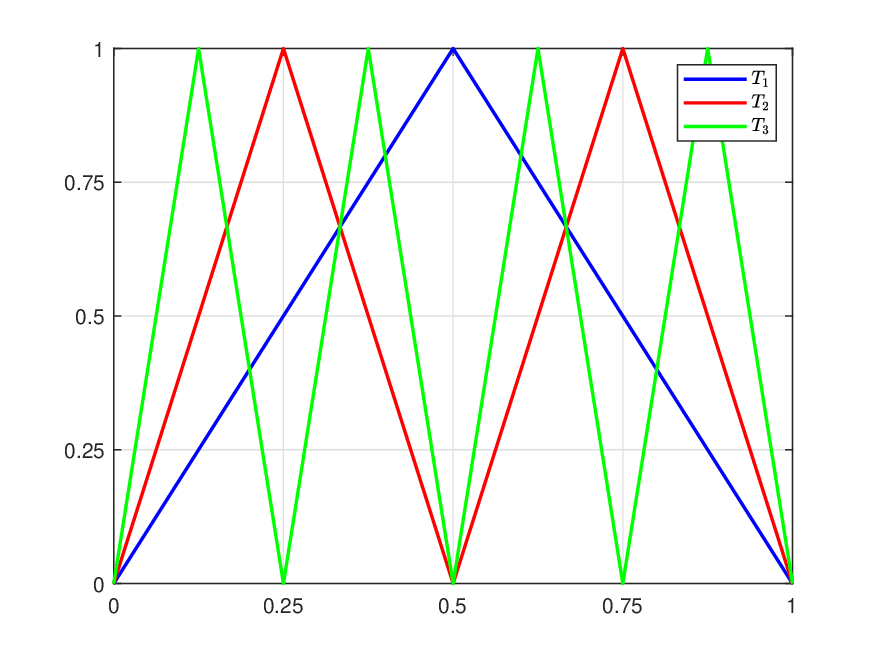}\label{Fig.2a}
      }
      \quad
      \subfigure[]{
           \includegraphics[scale=0.45]{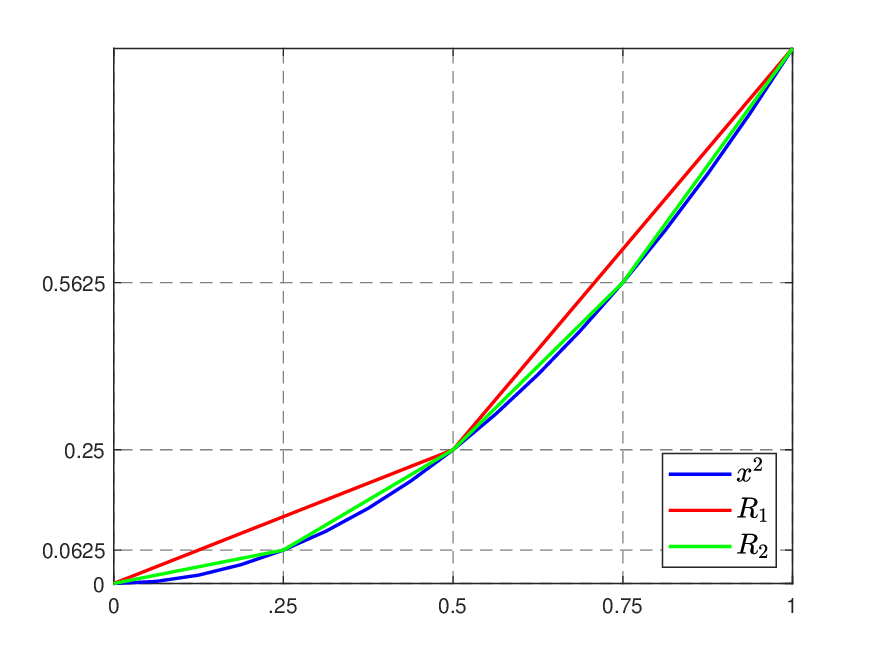}\label{Fig.2b}
      }   
    \caption{(a) Graphs of $T_1$, $T_2$, $T_3$; (b) interpolants $R_1$, $R_2$ of $x^2$.}\label{Fig.2}
\end{figure}

The approximate product $\widetilde{\times}_{M,U}(x,y)$ is then built upon $R_U$. The next lemma provides an error bound of $|\widetilde{\times}_{M,U}(x,y)-xy|$ (see \cite{Mao_Zhou2022}) as well as verifies non-negativity of $\widetilde{\times}_{M,U}(x,y)$ when $(x,y)\in[0,1]^2$. The non-negativity of $\widetilde{\times}_{M,U}(x,y)$ serves as a crucial property in our analysis and seems missing in the classical literature. %Furthermore, it also indicates that the upper bound of such a bivariate function is the square of the upper bound of the input variables, which is advantageous for constructing subsequent convolutional layers, and it provides the rate of approximation for the product. 
The proof of Lemma \ref{lemma:approxtimes} could be found in the appendix.
\begin{lemma}\label{lemma:approxtimes}
For any $U \in \mathbb N_{+}$ and $M>0$, 
let
$$
\widetilde{\times}_{M}(x,y)=\widetilde{\times}_{M,U}(x,y) :=2M^2\left[R_U\left(\frac{x+y}{2M}\right)-\frac{1}{4}R_U\left(\frac{x}{M}\right)-\frac{1}{4}R_U\left(\frac{y}{M}\right)\right].
$$
Then for any $x,y \in[0,M]$ it holds that 
$$
|\widetilde{\times}_{M,U}(x,y)-xy|\leq \frac{M^2}{2^{2U}},
\quad
0 \leq \widetilde{\times}_{M,U}(x,y)\leq M^2.
$$
\end{lemma}

Throughout the rest of this paper, the subscripts $M, U$ or $U$ in $\widetilde{\times}_{M,U}$ might be suppressed unless confusion arises. We shall prove the main approximation result for polynomials by developing the following three lemmata. %is similar to \cite[Lemma 2]{Mao_Zhou2022}, and we provide the following version. The distinction lies in the positioning of $R_U(\bm y)$ and $\bm y$ within the output vector of the CNN, which differs from the findings in \cite[Lemma 2]{Mao_Zhou2022}. This contrast will hold significance in the proof of Lemma \ref{lemma:vectorprod}. Moreover, the Lemma \ref{Le:approximation square} demonstrates that we can derive a CNN with depth $O(U)$ to approximate $x^2$ within accuracy $O(4^{-U})$. 

\begin{lemma}\label{lemma:squareapproximation} For any $L,U \in \mathbb N_{+}$, there exists $h_J$ of the form \eqref{def:hL} with depth $J\leq \frac{(7 U+15) L}{s-1}+3 U+2
$ such that for any $\bm{y} \in [0,1]^{L}$, we have
$$
\bm y \stackrel{h_{J}}\longrightarrow\left[R_U(\bm y);\mathbf{0}_{7 L};\bm y\right].
$$
\end{lemma}
The proof of Lemma \ref{lemma:squareapproximation} is left in the appendix.

\begin{lemma}\label{lemma:elimzeros}
Let $l,k,n \in \mathbb N_{+}$.  For any $\{\bm{y}_1, \ldots, \bm{y}_n\} \subset \mathbb R_{+}^k$, there exists $h_J$ of the form \eqref{def:hL} with depth
$
J \leq l \left\lceil\frac{(n-1)k}{s-1}\right\rceil
$ such that
$$
\left[\bm y_1;\bm {0}_{kl(n-1)};\bm y_2;\ldots;\bm y_n\right]\stackrel{h_{J}}\longrightarrow  [\bm y_1;\bm y_2;\ldots;\bm y_n]. 
$$
\end{lemma}
\begin{proof}
By taking $\bm w=[1;\bm {0}_{(n-1)k-1};1]$, $\bm b =\left[-\tilde b\cdot\bm 1;\bm{0}_{(l-1)(n-1)k+nk};-\tilde b\cdot\bm 1\right]$ with $\tilde b =\max_{1\leq i\leq n} \|\bm y_i\|_{\infty}$ and using Lemma \ref{Le:Represent_shallow_net},  we know that there exists $h_{J_1}$ with depth $J_1 \leq\left\lceil\frac{(n-1)k}{s-1}\right\rceil$ such that
$$\left[\bm y_1;\bm{0}_{l(n-1)k};\bm y_2;\ldots;\bm y_n\right] \stackrel{h_{J_1}}\longrightarrow \left[\bm y_1;\bm{0}_{(l-1)(n-1)k};\bm y_2;\ldots;\bm y_n\right].$$    Continuing this process $l$ times, we obtain a CNN with depth $J \leq l \left\lceil\frac{(n-1)k}{s-1}\right\rceil$.
\end{proof}
 
\begin{figure}[htbp]
\centering 
\includegraphics[width=6cm,height=5cm]{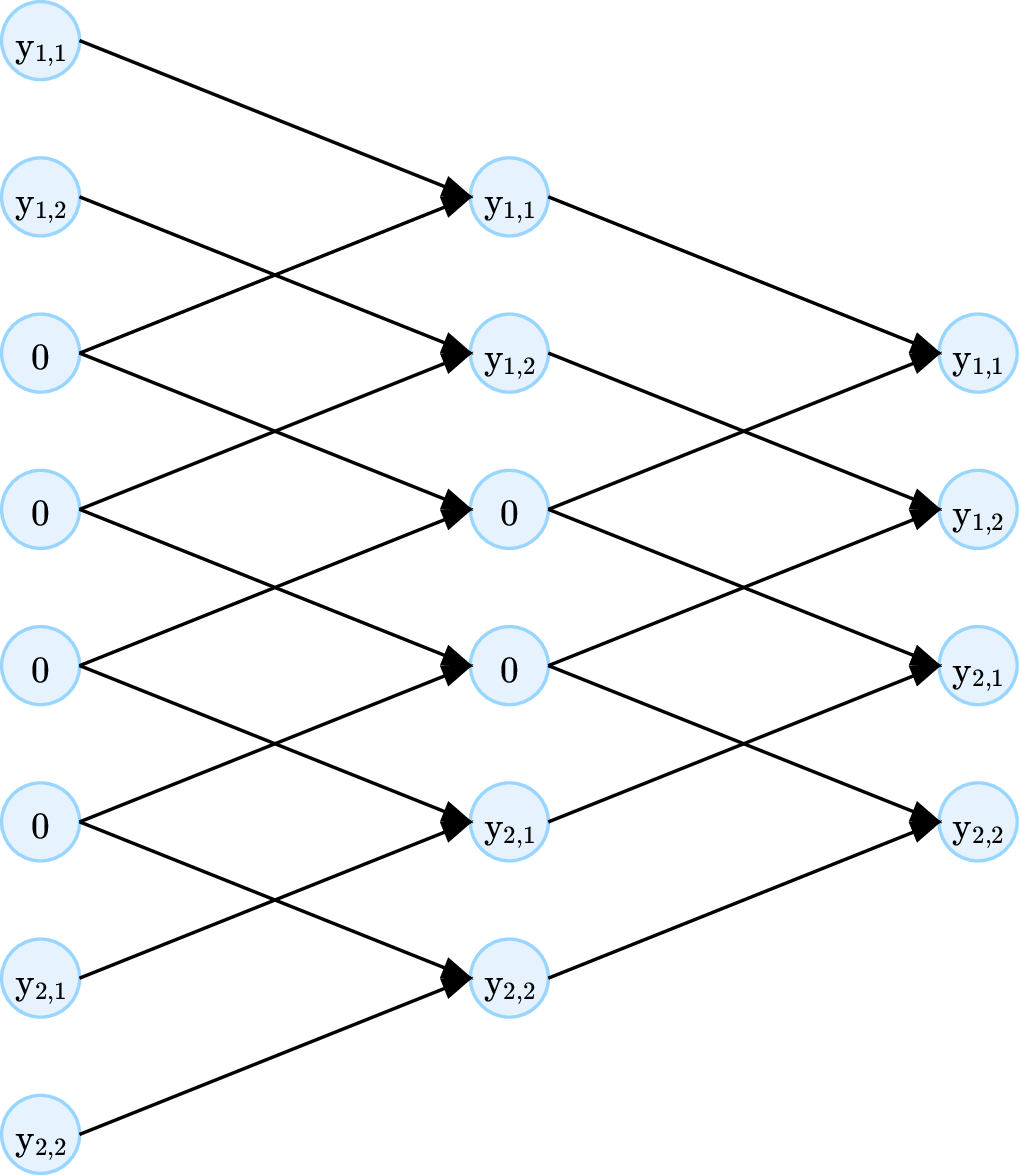}
\caption{An illustration of the proof of Lemma \ref{lemma:elimzeros},  $l=k=n=2$.}\label{sample-figure}
\end{figure}
The above lemma demonstrates that some CNN can eliminate zeros in the input vector. Lemma \ref{lemma:elimzeros} is useful for deriving the next lemma about approximately multiplying components in a partitioned vector by CNNs.
\begin{lemma}\label{lemma:vectorprod}
Let $U, k, l\in \mathbb N_{+}$ and $M>0$. For any $\{\bm{y}_1,\ldots,\bm{y}_l\} \subset \mathbb [0,M]^k$, there exists $h_J$ of the form \eqref{def:hL} with depth
$
J\leq \frac{(128+14U)lk}{s-1}+3U+61$ such that
$$[\bm y_1;\bm y_2;\bm y_3;\ldots;\bm y_l] \stackrel{h_{J}}\longrightarrow [\widetilde{\times}_{M,U}(\bm y_1,\bm y_2);\bm y_3;\ldots;\bm y_l],$$
where $\widetilde{\times}_{M,U}(\bm y_1,\bm y_2)$ denotes $\left(\widetilde{\times}_{M,U}(y_{1,1},y_{2,1}),\ldots,\widetilde{\times}_{M,U}(y_{1,k},y_{2,k})\right)^\top$. 
\end{lemma}
\begin{proof}
Taking $\bm w=\left[\frac{1}{2M};\bm{0}_{k-1};\frac{1}{2M};\bm{0}_{lk-1};\frac{1}{M}\right]$ and $\bm b = \left[-\bm{2};\bm 0_{k};-\bm{2}_{(l-1)k};\bm 0_{lk};-\bm{2}\right]$ in Lemma \ref{Le:Represent_shallow_net}, there exists $h_{J_1}$ with depth $J_1 \leq\left\lceil\frac{(l+1)k}{s-1}\right\rceil$ such that
\begin{equation*}
    [\bm y_1;\bm y_2;\bm y_3;\ldots;\bm y_l] \stackrel{h_{J_1}}\longrightarrow \left[\frac{\bm y_1+\bm y_2}{2M};\bm{0}_{(l-1)k};\frac{\bm y_1}{M};\ldots;\frac{\bm y_l}{M}\right].
\end{equation*}
Using $h_{J_1}$ and Lemma \ref{lemma:squareapproximation}, we obtain $h_{J_2}$ with $J_2 \leq J_1 +\frac{2(7 U+15)lk}{s-1}+3 U+2$ and
\begin{align*}
    &[\bm y_1;\bm y_2;\bm y_3;\ldots;\bm y_l]\stackrel{h_{J_2}}\longrightarrow 
    \left[R_U\left(\frac{\bm y_1+\bm y_2}{2M}\right);\bm{0}_{(l-1)k};R_U\left(\frac{\bm y_1}{M}\right);R_U\left(\frac{\bm y_2}{M}\right);\bm{0}_{17lk};\frac{\bm y_3}{M};\ldots;\frac{\bm y_l}{M}\right].
\end{align*}
Using Lemma \ref{lemma:approxtimes} and Lemma \ref{Le:Represent_shallow_net}  with $\bm w=\left[-\frac{M^2}{2};\bm{0}_{k-1};-\frac{M^2}{2};\bm{0}_{lk-1};2M^2;\bm{0}_{m};M\right]
$,
$\bm b = \left[\bm{0};\bm 0_{k};-b_1\cdot\bm{ 1}_{57(l-2)k};\bm{0}_{(l-2)k};\bm{0}\right]$, $m=39lk-115k-1$, $b_1=4M^2$ and $h_{J_2}$, we find $h_{J_3}$ with depth $J_3 \leq J_2 + \left\lceil\frac{40lk-114k}{s-1}\right\rceil$ such that
$$[\bm y_1;\bm y_2;\bm y_3;\ldots;\bm y_l]\stackrel{h_{J_3}}\longrightarrow\left[\widetilde{\times}_{M,U}(\bm y_1,\bm y_2);\bm{0}_{57(l-2)k};\bm y_3;\ldots;\bm y_l\right].$$ 
With the help of $h_{J_3}$ and Lemma \ref{lemma:elimzeros}, with $n=l-1, l=57$, we find $h_J$ with depth
$
J \leq J_3 + 57 \left\lceil\frac{(l-2)k}{s-1}\right\rceil
$ such that
$$[\bm y_1;\bm y_2;\bm y_3;\ldots;\bm y_l]\stackrel{h_{J}}\longrightarrow [\widetilde{\times}_{M,U}\left(\bm y_1,\bm y_2);\bm y_3;\ldots;\bm y_l \right].$$
Direct calculation shows 
$
J\leq \frac{(128+14U)lk}{s-1}+3U+61
$.
\end{proof}

% The following process explains the proof process of Lemma \ref{lemma:vectorprod} for $l=4$, $k=1$.
% \small
% \begin{align*}
% \begin{pmatrix}
% y_1 \\
% y_2  \\
% y_3  \\
% y_4  \\
% \end{pmatrix}
% &\xrightarrow[\text{layers}]{J_1}\begin{pmatrix}
% \frac{ y_1+ y_2}{2M}\\
% \bm{0}_{3}\\
% \frac{ y_1}{M}\\
% \frac{ y_2}{M}\\
% \frac{ y_3}{M}\\
% \frac{ y_4}{M}\\
% \end{pmatrix}\xrightarrow[\text{ layers}]{J_2 -J_1 }\begin{pmatrix}
% R_U(\frac{ y_1+ y_2}{2M})\\
% \bm{0}_{3}\\
% R_U(\frac{y_1}{M})\\
% R_U(\frac{y_2}{M})\\
% \bm{0}_{68}\\
% \frac{ y_3}{M}\\
% \frac{ y_4}{M}\\
% \end{pmatrix}\\
% &\xrightarrow[\text{ layers}]{J_3 -J_2 }\begin{pmatrix}
% \widetilde{\times}_{M}( y_1, y_2)\\
% \bm{0}_{114}\\
%  y_3\\
%  y_4\\
% \end{pmatrix}\xrightarrow[\text{ layers}]{J -J_3 }\begin{pmatrix}
% \widetilde{\times}_{M}( y_1, y_2)\\
%  y_3\\
%  y_4
% \end{pmatrix}.
% \end{align*}

The next theorem provides the approximation error bound by CNNs for approximating polynomials with non-negative bounded input variables.
\begin{theorem}\label{thm:polynomial}
Let $U,d,l,k\in \mathbb N_{+}$, $k\geq2$, $M>0$, $\bm c\in \mathbb R^{l}$, and $
\bm{y}=\left[\bm y_1;\bm y_2;\ldots;\bm y_k\right]\in\mathbb R^{dkl}
$ with
$
\bm y_i = [\bm 0_{d-1};y_{1,i}; \bm 0_{d-1}; y_{2,i} ;\ldots ,\bm 0_{d-1}; y_{l,i}]
$. 
Then there exist $\bar{\bm{c}}\in \mathbb R^{dkl +Js}$ and $h_J$ of the form \eqref{def:hL} 
such that for any $y_{i,j}\in[0,M]$ with $1\leq i\leq l, 1\leq j\leq k$,
\begin{equation*}
\Big|\sum_{i=1}^{l}c_{i}\prod_{j=1}^{k}y_{i,j}-\bar{\bm{c}} \cdot h_{J}(\bm{y})\Big|\leq\|\bm{c}\|_{\infty}\cdot \frac{M^{2^{k-1}}}{2^{2U-k+2}},
\end{equation*}  
where
$
J \leq \frac{(256+28U)dlk^2}{s-1}+(3U+61)k
$.
\end{theorem}
\begin{proof} 
By using Lemma \ref{lemma:vectorprod} $k-1$ times, we obtain CNNs $h_{J_1}, \ldots, h_{J_{k-1}}$ such that
\begin{align*}
\bm y &\stackrel{h_{J_1}}\longrightarrow \left[\widetilde{\times}_{M_1}\left(\bm y_1,\bm y_2\right);\ldots;\bm y_k\right]\stackrel{h_{J_2}}\longrightarrow \left[\widetilde{\times}_{M_2}\left(\widetilde{\times}_{M_1}\left(\bm y_1,\bm y_2\right),\bm y_3\right);\ldots;\bm y_k \right]\\
\cdots&\stackrel{h_{J_{k-1}}}\longrightarrow \left[\widetilde{\times}_{M_{k-1}}\left(\widetilde{\times}_{M_{k-2}}\left(\cdots \right);\bm y_k \right)\right]
:=[\bm 0_{d-1};g_1;\bm 0_{d-1};g_2;\ldots;\bm 0_{d-1};g_l], 
\end{align*}
where $M_{j} = M^{2^{j-1}}$ and $J_{j}\leq \frac{(128+14U)(k-j+1)dl}{s-1}+3U+61$. We claim that 
\begin{equation}\label{giyij_error}
\Big|g_i-\prod_{j=1}^{k}y_{i,j}\Big|\leq \frac{M^{2^{k-1}}2^{k-2}}{2^{2U}},
\end{equation}
and prove it by induction on $k$. The case $k=2$ follows from Lemma \ref{lemma:approxtimes}:
\begin{equation*}
|\widetilde{\times}_{M_1}(y_{i,1},y_{i,2})-y_{i,1}y_{i,2}|\leq \frac{M^2}{2^{2U}},\quad0 \leq \widetilde{\times}_{M_1}(y_{i,1},y_{i,2})\leq M^2.
\end{equation*}
Next assume the following is true for some $k\geq2$:
\begin{equation}\label{eqn:induction}
\begin{aligned}
\Big|\widetilde{\times}_{M_{k-1}}\left(\widetilde{\times}_{M_{k-2}}(\cdots ), y_{i,k}\right)-\prod_{j=1}^ky_{i,j}\Big|&\leq \frac{M^{2^{k-1}}2^{k-2}}{2^{2U}},\\
\widetilde{\times}_{M_{k-1}}\left(\widetilde{\times}_{M_{k-2}}(\cdots ), y_{i,k}\right)&\leq M^{2^{k-1}}.
\end{aligned}
\end{equation}
We need to show that \eqref{eqn:induction} still holds when $k$ is replaced with $k+1$. To this end, combining Lemma \ref{lemma:approxtimes} and \eqref{eqn:induction} yields
\begin{align*}
&\Big|\widetilde{\times}_{M_k}\left(\widetilde{\times}_{M_{k-1}}(\cdots ), y_{i,k+1}\right)-\prod_{j=1}^{k+1}y_{i,j}\Big|\\
&\leq\left|\widetilde{\times}_{M_k}\left(\widetilde{\times}_{M_{k-1}}(\cdots ),y_{i,k+1}\right)-\widetilde{\times}_{M_{k-1}}\left(\widetilde{\times}_{M_{k-2}}(\cdots ),y_{i,k}\right)\cdot y_{i,k+1}\right|\\
&+\Big|\widetilde{\times}_{M_{k-1}}\left(\widetilde{\times}_{M_{k-2}}(\cdots ),y_{i,k}\right)-\prod_{j=1}^ky_{i,j}\Big|\cdot y_{i,k+1}\\
&\leq \frac{(M^{2^{k-1}})^2}{2^{2U}}+M \cdot \frac{M^{2^{k-1}}2^{k-2}}{2^{2U}}\leq \frac{M^{2^k}2^{k-1}}{2^{2U}}
\end{align*}
and thus completes the induction. We set 
$h_J:=h_{J_{k-1}}\circ h_{J_{k-2}}\circ \cdots \circ h_{J_1}$
with 
$$
J =\sum_{j=1}^{k-1}J_j\leq \frac{(256+28U)dlk^2}{s-1}+(3U+61)k.
$$
It follows from \eqref{giyij_error} and $\bar {\bm c}:=[\bm{0};\bm 0_{d-1};c_1;\bm 0_{d-1};c_2;\ldots;\bm 0_{d-1};c_l;\bm{0}]\in \mathbb R^{dkl +Js}$ that 
$$
\Big|\sum_{i=1}^{l}c_{i}\prod_{j=1}^{k}y_{i,j}-\bar{\bm{c}} \cdot h_{J}(\bm y)\Big|\leq \|\bm{c}\|_{\infty}\cdot \frac{M^{2^{k-1}}}{2^{2U-k+2}}.
$$
The proof is complete.
\end{proof}
We can observe that for any multivariate function $\Phi(\bm x)$ of the form 
\begin{equation}\label{Eq:Phi}
\Phi(\bm x)=\sum_{i=1}^{l}c_{i}\prod_{j=1}^{k}\phi_{i,j}(x_{i,j}),
\end{equation}
where $ x_{i,j}\in\{x_1,x_2,\ldots,x_d\}$ for $ \bm x\in [0,1]^d$,
as long as we can construct a CNN such that for the input variable $\bm x$, the output vector of the CNN is 
$$
\bm{\phi(\bm x)}=\left[\bm \phi_1;\bm \phi_2;\ldots;\bm \phi_k\right],
$$
where each $\bm \phi_i$ is of the form
$$
\bm \phi_i = [\bm 0_{d-1};\phi_{1,i}; \bm 0_{d-1}; \phi_{2,i} ;\ldots ;\bm 0_{d-1}; \phi_{l,i}],
$$
and each $ \phi_{i,j}$ is a non-negative univariate function with an upper bound of $M$ in \eqref{Eq:Phi} for $i=1,2,\ldots,l$ and $j=1,2,\ldots,k$,
then we can construct a  CNN regarding $\bm \phi(\bm x)$ as the input vector to approximate the function $\Phi(\bm x)$.  Therefore, we can use Theorem \ref{thm:polynomial} to construct a CNN to approximate a function  by using form \eqref{Eq:Phi} as an intermediate function for a specific function space. Specifically, when each $ \phi_{i,j} $ is composed of fully connected neural network functions, $ \Phi(\bm x) $  is called the tensor neural network introduced by \cite{li_Lin_Wang_Xie2024TNN}.  To the best of our knowledge, there is no result given for the approximation of multivariate polynomials and cardinal B-splines using CNNs in classical literature. In fact, we can observe that both multivariate polynomials
and cardinal B-splines have the form of \eqref{Eq:Phi}.
%It was shown in \cite[Theorem 1.9]{Liang_Srikant2017} that the general multivariate polynomials $\sum_{\bm{\alpha}:|\bm{\alpha}| \leq n^{\prime}} C_{\bm{\alpha}} \bm{x}^{\bm{\alpha}}$ for $ \bm{x} \in[0,1]^d$ and $\sum_{\bm{\alpha}:|\bm{\alpha}| \leq n^{\prime}}\left|C_{\bm{\alpha}}\right| \leq 1$ can be approximated with accuracy $\varepsilon$ by a DNN with depth $O\left(n^{\prime}+\log \frac{d n^{\prime}}{\varepsilon}\right)$. 
%By using Theorem \ref{thm:polynomial}, one can show that this process can be replaced by a CNN defined in this paper with depth
%$O\left(dn^{\prime}+\log \frac{1}{\varepsilon}\right)$. 
For example, in \cite{HeMaoXu2023expressivity}, the authors present the approximation rates for Sobolev and analytic functions achieved by representing polynomials with deep $\text{ReLU}^k$ (the $k$-th power of the ReLU activation function) neural networks, combined with existing polynomial approximation results (see \cite[Chapter 7]{DeVoreLorentz1993}). Therefore, utilizing Lemma \ref{Le:Represent_shallow_net} and Theorem \ref{thm:polynomial}, one can immediately obtain the corresponding approximation rates for Sobolev and analytic functions by CNNs.

\section{Analysis of Approximation Error Bounds of CNNs}\label{sec:proof}
This section is devoted to the proof of Theorem \ref{thm:mainresult}. 
%To approximate the sparse grid interpolant $f_n$ for a sufficiently large $n$ by CNNs, we need to represent each basis function $\phi_{\bm l, \bm {i}}^{\bm{\alpha}}(\bm x)$ approximately using CNNs. 
%We have constructed the CNNs to approximate $xy$ in Section \ref{sec:polyapproximation}. Motivated by the work \cite{Mao_Zhou2022}, a family of functions of the form $\sigma(ax+b)$ can be precisely represented in the output vector of the last layer of a CNN by combining Lemma \ref{Le:Represent_shallow_net}.
As explained in Section \ref{sec:polyapproximation}, our strategy is based on Theorem \ref{thm:polynomial} and the precise ReLU
product factors decomposition of each $\phi_{\bm l,\bm i}^{\bm \alpha}$.
\begin{proof}[Proof of Theorem \ref{thm:mainresult}]
Let $\Sigma_n:=\left\{(\bm l,\bm i):|\bm{l}| \leq n+d-1,\bm{i} \in \mathcal{I}_{\bm l}\right\}=\{\bm{\beta}_1, \ldots, \bm{\beta}_N\}$ with $N:=|\Sigma_n|=O(2^n n^{d-1})$. Given $m\geq2$, $n\in\mathbb{N}_+$, consider 
$$
I_nf(\bm{x})=\sum_{|\bm l|_1 \leq n+d-1} \sum_{\bm i \in \mathcal{I}_{\bm l}} v_{\bm l, \bm i} \phi_{\bm l, \bm i}^{\bm \alpha}(\bm {x}):=\sum_{\bm{\beta}=(\bm{l},\bm{i})\in\Sigma_n}v_{\bm{\beta}} \phi_{\bm{\beta}}^{\bm \alpha}(\bm {x}),
$$ 
where $\phi_{\bm l, \bm i}^{\bm \alpha}(\bm {x})$ will be  represented as a product of piecewise linear functions. 

We first consider the case $m = 2$, $|\bm l|_1 \leq n+d-1$, which implies $\bm{\alpha}=\bm 2$. Let
$$
\rho_{l_j, i_j,k}\left(x\right):=\sigma\left(\frac{x-x_{l_j,i_j}+(-1)^{k-1}h_{l_j}}{(-1)^{k-1}h_{l_j}}
\right),\quad k=1,2.
$$
It follows from \eqref{eqn:phi2} that
\begin{equation*}
\phi_{l_j, i_j}^{ 2}\left(x\right)=\rho_{l_j,i_j,1}\left(x\right)\rho_{l_j,i_j,2}\left(x\right),
\end{equation*}
see Figure \ref{Fig.4}. When $x\in[0,1]$, we have $|x-x_{l_j,i_j}\pm h_{l_j}|\leq 1$ and 
$$0\leq\rho_{l_j, i_j,k}\left(x\right)\leq 2^{n+d-1},\quad k=1,2.$$
  
\begin{figure}[htbp]
      \centering
      \subfigure[]{
           \includegraphics[scale=0.45]{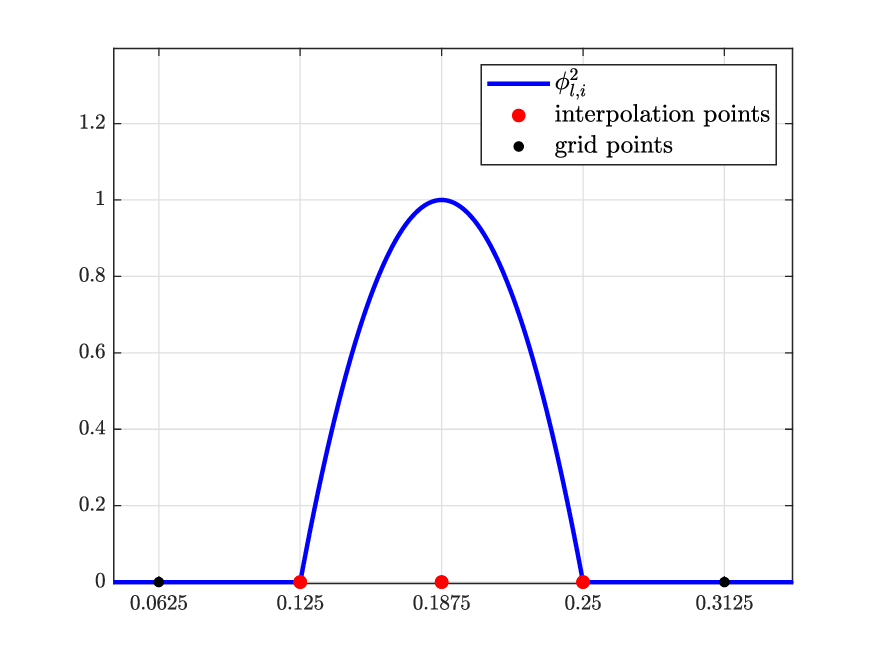}\label{Fig.4a}
      }
      \quad
      \subfigure[]{
           \includegraphics[scale=0.45]{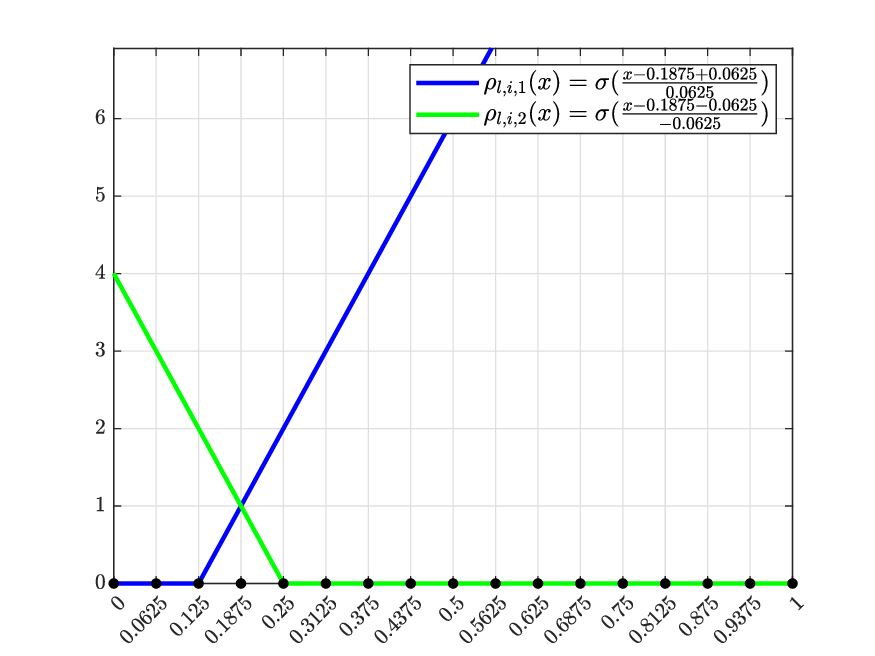}\label{Fig.4b}
      }   
    \caption{(a) $d=1$, $l=4$, $i=3$, the basis function $\phi_{l,i}^2$ at grid point $x_{l,i} = 0.1875$; (b) factors $\rho_{l,i,1}$, $\rho_{l,i,2}$ of $\phi_{l,i}^2$.}\label{Fig.4}
\end{figure}
In general, when $m\geq3$ and $ |\bm{l}|_1 \leq n+d-1$, if $\alpha_j = m$ for some $j$, we set 
\begin{equation}\label{Eq:rholjk}
\rho_{l_j,i_j,k}(x):=\sigma\left(\frac{x-x_{l_j,i_j,k}}{x_{l_j,i_j}-x_{l_j,i_j,k}}\right),\quad k=1,\ldots,m.
\end{equation}
see Figure \ref{Fig.5}. In case of $\alpha_j < m$, we set
$
\rho_{l_j,i_j,k}(x):=1
$ for $\alpha_j < k \leq m$ and set $
\rho_{l_j,i_j,k}(x)
$ as in \eqref{Eq:rholjk} for $1 \leq k \leq \alpha_j$. It follows from \eqref{eqn:phialpha_j} that 
$$
\phi_{l_j, i_j}^{\alpha_j}(x)=\prod_{k=1}^m \rho_{l_j,i_j,k}(x).
$$
When $x \in [0,1]$, we have
$|x-x_{l_j,i_j,k}| \leq 1, |x_{l_j,i_j} - x_{l_j,i_j,k}| \geq h_{l_j}$ and
$$ 
0 \leq \rho_{l_j, i_j,k}(x) \leq 2^{n+d-1},\quad k=1,\ldots,m.
$$ 
In summary, for $ m \geq 2 $, the basis function $\phi_{\bm l, \bm {i}}^{\bm{\alpha}}(\bm x)$ is of the form
$$
\phi_{\bm{l}, \bm{i}}^{\bm \alpha}(\bm{x})=\prod_{j=1}^d \phi^{\alpha_j}_{l_j, i_j}\left(x_j\right)=\prod_{j=1}^d \prod_{k=1}^m \rho_{l_j,i_j,k}(x_j).
$$

\begin{figure}[htbp]
      \centering
      \subfigure[]{
           \includegraphics[scale=0.45]{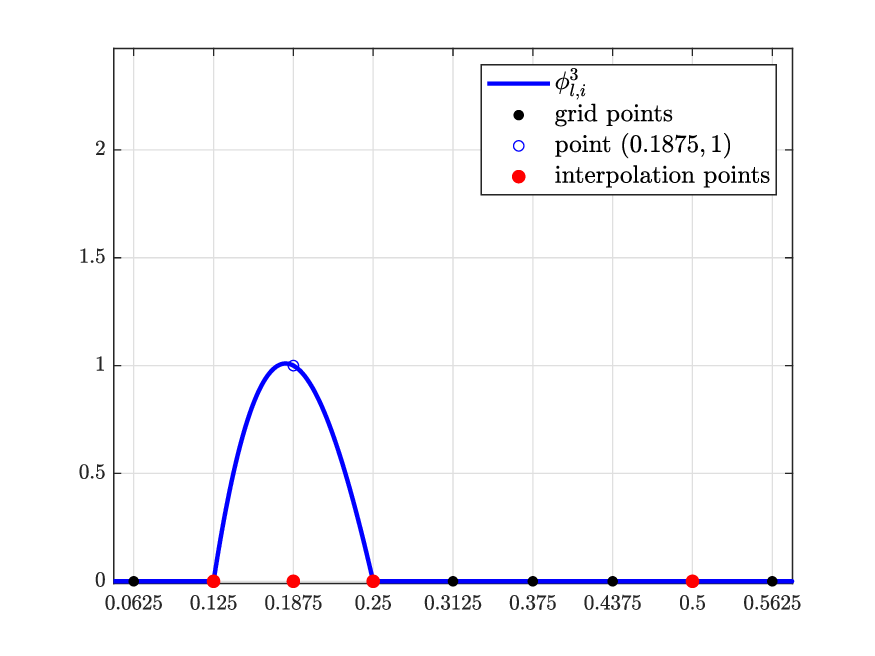}\label{Fig.5a}
      }
      \quad
      \subfigure[]{
           \includegraphics[scale=0.45]{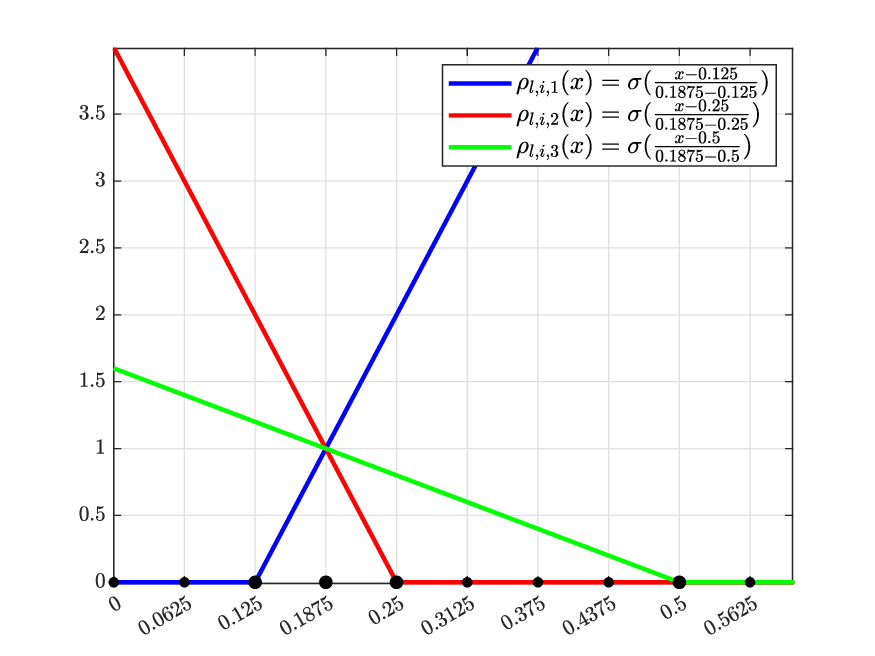}\label{Fig.5b}
      }   
    \caption{(a) $d=1$, $l=4$, $i=3$, the basis function $\phi_{l,i}^3$ at grid point $x_{l,i} = 0.1875$; (b) factors $\rho_{l,i,1}$, $\rho_{l,i,2}$, $\rho_{l,i,3}$ of $\phi_{l,i}^3$.}\label{Fig.5}
\end{figure}

Consider the set of all piecewise linear factors:
$$
\big\{\rho_{\bm{\beta}_i,k}(x_j): i=1, \ldots, N,~j=1,\ldots,d,~k=1,\ldots,m\big\},
$$
where each $\rho_{\bm{\beta}_i,k}(x_j)$ can be written as the form $ \sigma(a_{\bm{\beta}_i,j,k}x_j+b_{\bm{\beta}_i,j,k})$ for $a_{\bm{\beta}_i,j,k}, b_{\bm{\beta}_i,j,k}\in\mathbb{R}$ and $0\leq\rho_{\bm{\beta}_i,k}(x_j)\leq 2^{n+d-1}$. Let
\begin{align*}
    \bm{w}&:=\left[\bm w_{1,1};\ldots;\bm w_{1,m};\ldots;\bm w_{d,1};\ldots;\bm w_{d,m}\right],\\
\bm{b}&:=\left[\bm b_{1,1};\ldots;\bm b_{1,m};\ldots;\bm b_{d,1};\ldots;\bm b_{d,m};{\bm 0}_{d-1}\right]
\end{align*}
where each entry is given by  
\begin{align*}
\bm w_{j,k}&=\left[\mathbf{0}_{d-j};
\frac{a_{\bm{\beta}_1,j,k}}{2^{n+d-1}};
\mathbf{0}_{j-1};\mathbf{0}_{d-j};
\frac{a_{\bm{\beta}_2,j,k}}{2^{n+d-1}};
\mathbf{0}_{j-1};\ldots;\mathbf{0}_{d-j};
\frac{a_{\bm{\beta}_N,j,k}}{2^{n+d-1}};
\mathbf{0}_{j-1} \right],\\
\bm b_{j,k}&=\left[-2^{n+d} \mathbf{1}_{d-1};
\frac{b_{\bm{\beta}_1,j,k}}{2^{n+d-1}};-2^{n+d} \mathbf{1}_{d-1};
\frac{b_{\bm{\beta}_2,j,k}}{2^{n+d-1}};\ldots;-2^{n+d} \mathbf{1}_{d-1};
\frac{b_{\bm{\beta}_N,j,k}}{2^{n+d-1}} \right]. 
\end{align*}
Direct calculation shows that
$$
\sigma\left(
T_{\bm w}{\bm x}+{\bm b} 
\right)=[\bm y_{1,1};\ldots;\bm y_{1,m};\ldots;\bm y_{d,1};\ldots;\bm y_{d,m};\bm 0_{d-1}],
$$
where  
$\bm y_{j,k}=2^{-n-d+1}[{\bm 0}_{d-1};\rho_{\bm{\beta}_1,j,k}(x_j);{\bm 0}_{d-1};\rho_{\bm{\beta}_2,j,k}(x_j);\ldots;{\bm 0}_{d-1};\rho_{\bm{\beta}_N,j,k}(x_j)].$
It then follows from Lemma \ref{Le:Represent_shallow_net}  for $\bm w$ with filter length $m d^2 N-1$ and $\bm b\in \mathbb R^{m d^2 N+d-1}$ that there exists $h_{J_1}$ with depth $J_1 \leq \left\lceil\frac{m d^2 N-1}{s-1}\right\rceil$ satisfying 
\begin{equation*}
   \bm x \stackrel{h_{J_1}}\longrightarrow \bm{y}:=[\bm y_{1,1};\ldots;\bm y_{1,m};\ldots;\bm y_{d,1};\ldots;\bm y_{d,m}]. 
\end{equation*}

Using the output $\bm{y}$ of $h_{J_1}$ as the input of a next convolutional layer and Theorem \ref{thm:polynomial} with $l=N,k=dm$, $M=1$, $ U=\lceil md\log_2 N +dm \rceil$, we could construct a CNN function $\tilde{f}_L\in \mathcal{H}^{s,d}_L$ with depth 
$$
L \leq J_1 +\frac{(256+28U)Nd^3m^2}{s-1}+(3U+61)dm\leq C_sd^4m^3 N \log_2 N
$$
such that the following bound holds:
\begin{equation*}
\Big|\sum_{|\bm{l}|_1 \leq n+d-1} \sum_{\bm{i} \in \mathcal{I}_{\bm l}} v_{\bm{l}, \bm{i}} \cdot \frac{\phi_{\bm{l}, \bm{i}}^{\bm{\alpha}}(\bm x)}{2^{dm(n+d-1)}}-\tilde{f}_L\Big|
\leq  \max _{|\bm{l}|_1 \leq n+d-1,\bm{i} \in \mathcal{I}_{\bm l}}\left|v_{\bm{l}, \bm{i}}\right|\cdot \frac{2^{dm}}{2^{2U+2}}.
\end{equation*}
Combining the previous estimate with Lemma \ref{Le:coeffient_p_bound}, we have 
\begin{equation}\label{Eq:DCNNerror}
\begin{aligned}
&|I_nf-2^{dm(n+d-1)} \tilde{f}_L|\\
&\leq 2^{dm(n+d-1)}\Big|\sum_{|\bm{l}|_1 \leq n+d-1} \sum_{\bm{i} \in \mathcal{I}_{\bm l}} v_{\bm{l}, \bm{i}} \cdot \frac{\phi_{\bm{l}, \bm{i}}^{\bm{\alpha}}(\bm x)}{2^{dm(n+d-1)}}-\tilde{f}_L\Big|\\
&\leq \max _{|\bm{l}|_1 \leq n+d-1,\bm{i} \in \mathcal{I}_{\bm l}}\left|v_{\bm{l}, \bm{i}}\right|\cdot \frac{2^{dm(n+d)}}{2^{2U+2}}\\
&\leq C_{m,d}\left\|D^{\bm{\alpha}+\bm{1}} f\right\|_{L_p(\Omega)} N^{-m-1}.
\end{aligned}
\end{equation}
Setting $f_L:=2^{md(n+d-1)}\tilde{f}_L \in \mathcal{H}^{s,d}_L$ and using \eqref{Eq:DCNNerror} and Lemma \ref{lemma:interpolationerror}, we obtain
$$
\begin{aligned}
\|f-f_L\|_{L_p(\Omega)} &\leq C_{m,d} \|D^{\bm{\alpha}+\bm{1}} f\|_{L_p(\Omega)} N^{-m-1} \left(\left(\log _2 N\right)^{ (m+2)(d-1)}+ C_{m,d}\right)\\
&\leq C_{m,d} \|D^{\bm{\alpha}+\bm{1}} f\|_{L_p(\Omega)} N^{-m-1} \left(\log _2 N\right)^{ (m+2)(d-1)}, 
\end{aligned}
$$
which completes the proof.
\end{proof}

\begin{remark} In addition to approximation errors for target functions by neural networks, convergence rates of learning with neural networks have also been extensively investigated in the literature (cf. \cite{YangZHou2024nonparametric, YangZhou2025shallowreluk, ZhouHuo2024learning, LinWangWangZhou2022universal}). Following the line of \cite{YangZhou2025shallowreluk} regarding CNNs and utilizing the approximation error bounds derived in this section, one can establish the convergence rates for learning a Korobov function $f \in K^m_p$ with $p, m \geq 2$ by the CNN model $\mathcal{H}^{s,d}_L$. %Specifically, for the CNN model $\mathcal{H}^{s,d}_L$ with depth $L = C n^{\frac{1}{2(m+1)}}$, the convergence rate can be shown as $O\left( n^{-\frac{m}{m+1}} (\log n)^{2d(m+1)} \right)$.
\end{remark}

\section{Concluding Remarks}\label{sec:conclusion}
We have established higher order approximation rates $O(L^{-m-1})$ (up to logarithmic factors) for Korobov functions in $K_p^{m+1}(\Omega)$ by deep CNNs with depth $O(L\log_2 L)$. Currently, we are not sure whether the proposed  approximation error bound is nearly optimal. %Furthermore, we observe in \cite{Elbrachter_Grohs_Jentzen_Schwab2022} that a class of multivariate solutions of Kolmogorov equations can be effectively represented using DNNs. As shown in \cite[Proposition 2.1]{Elbrachter_Grohs_Jentzen_Schwab2022}, the solutions to these equations have a tensor product structure. If  all 
%$h_{c,K}$ there can be effectively represented by CNNs, then we can derive the corresponding CNN expression rate of high-dimensional PDEs. It would be interesting
%to see if Theorem \ref{Th:Combination products} can be used to realize the desired results. 
With the help of sparse grids and bit-extraction technique which was first introduced in \cite{anthony2009neural}, we note that \cite{Yang_Lu2024} recently provided super-approximation rate $O(N^{-4}L^{-4})$ (up to logarithmic factors) of Korobov functions in $K_{\infty}^2(\Omega)$ by DNNs with width $O(N(\log_2 N)^{d+1})$ and depth $O(L(\log_2 L)^{d+1})$. It would be interesting and promising to improve the approximate rates of deep CNNs in the current manuscript by utilizing the bit-extraction technique in \cite{Bartlett_Harvey_Liaw_Mehrabian2019,Yang2024,Siegel2023JMLR}. 

\section*{Appendix}\label{appendix:A}

\subsection*{Proof of Lemma \ref{Le:coeffient_p_bound}.}
\begin{proof}
The proof of Lemma 4.6 in \cite{Bungartz_Griebel2004} includes the following:
\begin{equation}\label{eqn:ws}
\begin{aligned}
\Big|\frac{w_{l_j, i_j}^{\prime}\left(x_{l_j, i_j}\right)}{{\alpha}_{j}!}\int_{[0,1]} s_{l_j, i_j}^{{\alpha}_j}\left(x_j\right) \mathrm{d} x_j\Big| &\leq \frac{1}{\left({\alpha}_j+1\right)!} \cdot h_{l_j}^{{\alpha}_j+1} \cdot 2^{{\alpha}_j \cdot\left({\alpha}_j+1\right) / 2-1}, \\
|w_{l_j, i_j}^{\prime}(x_{l_j, i_j})| &\leq h_{l_j}^{{\alpha}_j+1} 2^{{\alpha}_j \cdot\left({\alpha}_j+1\right) / 2-1},\\
\Big|s_{l_j, i_j}^{{\alpha}_j}\left(x_j\right)\Big| &\leq \frac{2}{3} \cdot \frac{1}{\left({\alpha}_j+1\right) \cdot h_{l_j}}
\end{aligned}
\end{equation}
for $x_j \in [0,1]$, $j=1,\ldots,d$. It then follows from  \eqref{eqn:ws} that
$$
\begin{aligned}
&\left\|g_{\bm{l}, \bm{i}}^{\bm{\alpha}}\right\|_{L_p(\Omega)}=\Big(\int_{\Omega}\Big|\prod_{j=1}^d\frac{w_{l_j, i_j}^{\prime}\left(x_{l_j, i_j}\right)}{\alpha_j!} s_{l_j, i_j}^{\alpha_j}\left(x_j\right)\Big|^p \mathrm{~d}\bm  x\Big)^{1 / p}\\
&=\prod_{j=1}^d\Big(\int_{[0,1]}\Big|\frac{w_{l_j, i_j}^{\prime}\left(x_{l_j, i_j}\right)}{{\alpha}_{j}!} s_{l_j, i_j}^{{\alpha}_j}\left(x_j\right)\Big|^p \mathrm{~d} x_j\Big)^{1 / p} \\
&\leq \prod_{j=1}^d\Big(\max _{x_j \in[0,1]}\Big|\frac{w_{l_j, i_j}^{\prime}\left(x_{l_j, i_j}\right)}{{\alpha}_{j}!} s_{l_j, i_j}^{{\alpha}_j}\left(x_j\right)\Big|^{p-1} \cdot \frac{1}{\left({\alpha}_j+1\right)!} h_{l_j}^{{\alpha}_j+1} 2^{{\alpha}_j \cdot\left({\alpha}_j+1\right) / 2-1}\Big)^{1 / p} \\
&\leq \prod_{j=1}^d\left(\left|\frac{2}{3} \cdot \frac{1}{  h_{l_j}}\right|^{p-1}  \right)^{1 / p}\cdot \frac{h_{l_j}^{{\alpha}_j+1} 2^{{\alpha}_j \cdot\left({\alpha}_j+1\right) / 2-1}}{\left({\alpha}_j+1\right)!}\leq c(\bm{{\alpha}}) \cdot 2^{-d-|\bm{l} \odot\bm{\alpha}|_1-\frac{|\bm l|_1}{p}}.
\end{aligned}
$$
Combining the above bound with Lemma \ref{lemma:vlialpha}, we have
\begin{equation*}
\begin{aligned}
|v_{ \bm l, \bm i}|&=\Big|\int_{\Omega} g_{ \bm l,  \bm i}^{\bm {\alpha}}(\bm x) \cdot D^{\bm{{\alpha}}+\bm{1}} f(\bm{x}) \mathrm{d}\bm{x}\Big|\\
&\leq \left\|g_{\bm{l}, \bm{i}}^{\bm{\alpha}}\right\|_{L_{p^{\prime}}(\Omega)}\left\|D^{\bm{\alpha}+\bm{1}} f\right\|_{L_p(\operatorname{supp}(\phi_{\bm{l}, \bm{i}}^{\bm \alpha}))}\\
&\leq c(\bm{\alpha}) \cdot 2^{-d-|\bm{l} \odot\bm{\alpha}|_1-\frac{|\bm{l}|_1}{p^{\prime}}}\cdot\left\|D^{\bm{\alpha}+\bm{1}} f\right\|_{L_p(\operatorname{supp}(\phi_{\bm{l}, \bm{i}}^{\bm \alpha}))}.
\end{aligned} 
\end{equation*}
The proof is complete.    
\end{proof}

\subsection*{Proof of Lemma \ref{lemma:interpolationerror}.}

\begin{lemma}
\cite[Lemma 3.7]{Bungartz_Griebel2004}\label{lemma:sum_l}
For any $n, t \in \mathbb{N}_{+}$, we have
$$
\sum_{|\bm l|_1>n+d-1} 2^{-t|\bm l|_1} 
\leq 2^{-tn-td-1}A(d, n),
$$  
where $A(d, n):=\sum_{k=0}^{d-1}\binom{n+d-1}{k}=\frac{n^{d-1}}{(d-1)!}+O\left(n^{d-2}\right) .$
\end{lemma}

\begin{proof}
For $1 \leq p <\infty$,
$$
\|f-I_nf\|_{L_p(\Omega)} \leq \sum_{|\bm{l}|_1 > n+d-1} \Big\|\sum_{\bm{i} \in \mathcal{I}_{\bm l}} v_{\bm{l}, \bm{i}} \phi_{\bm{l}, \bm{i}}^{\bm{\alpha}}(\bm x)\Big\|_{L_p(\Omega)}.
$$
It follows from ${\rm supp}\big(\phi_{\bm l,\bm i}^{\bm \alpha}\big) \bigcap {\rm supp}\big(\phi_{\bm l, \bm{i}^{\prime}}^{\bm \alpha}\big)=\varnothing$ for $\bm{i} \neq \bm{i}^{\prime}$, Lemmata \ref{lemma:philialpha_Lpbound} and \ref{Le:coeffient_p_bound} that
$$
\begin{aligned}
\int_{\Omega}\Big|\sum_{\bm i \in \mathcal{I}_{\bm l}} v_{\bm{l}, \bm{i}} \phi_{\bm{l}, \bm{i}}^{\bm \alpha}\Big|^p d \bm x
&=\sum_{\bm i \in \mathcal{I}_{\bm l}} \int_{\operatorname{supp}\left(\phi_{\bm{l}, \bm{i}}^{\bm \alpha}\right)}\left|v_{\bm{l}, \bm{i}}\phi_{\bm{l}, \bm{i}}^{\bm \alpha}\right|^p d \bm x \\
&\leq 1.117^{pd} \cdot 2^{d-|\bm l|_1}\sum_{\bm i \in \bm{i}_{\bm l}} |v_{\bm{l}, \bm{i}}|^p\\
&\leq c(\bm{\alpha})^p \cdot 2^{-p|\bm{l} \odot\bm{\alpha}|_1-\frac{p|\bm l|_1}{p^{\prime}}-|\bm{l}|_1}\cdot\sum_{\bm{i} \in \mathcal{I}_{\bm l}}\left\|D^{\bm{\alpha}+\bm{1}} f\right\|_{L_p(\operatorname{supp}(\phi_{\bm{l}, \bm{i}}^{\bm \alpha}))}^p\\
&\leq c(\bm{\alpha})^p \cdot 2^{-p|\bm{l} \odot(\bm{\alpha+1})|_1}\cdot\|D^{\bm{\alpha}+\bm{1}} f\|_{L_p(\Omega)}^p.
\end{aligned}
$$
As a result, we obtain
\begin{equation}\label{eqn:sum_Lp}
    \Big\|\sum_{\bm{i} \in \mathcal{I}_{\bm l}} v_{\bm{l}, \bm{i}} \phi_{\bm{l}, \bm{i}}^{\bm{\alpha}}(\bm x)\Big\|_{L_p(\Omega)}\leq c(\bm{\alpha}) \cdot 2^{-|\bm{l} \odot(\bm{\alpha+1})|_1}\cdot\|D^{\bm{\alpha}+\bm{1}} f\|_{L_p(\Omega)},
\end{equation}
Combining Lemma \ref{lemma:sum_l} with \eqref{eqn:sum_Lp} leads to
\begin{equation*}
\begin{aligned}
\|f-I_nf\|_{L_p(\Omega)} &\leq \sum_{|\bm{l}|_1 > n+d-1} c(\bm{\alpha}) \cdot 2^{-|\bm{l} \cdot(\bm{\alpha+1})|_1}\cdot\|D^{\bm{\alpha}+\bm{1}} f\|_{L_p(\Omega)}\\
& \leq 
c(\bm{\alpha}) \cdot 2^{-(m+1)n}\cdot n^{d-1}\cdot \|D^{\bm{\alpha}+\bm{1}} f\|_{L_p(\Omega)}\\
&\leq  C_{m,d} \|D^{\bm{m}+\bm{1}} f\|_{L_p(\Omega)} N^{-m-1} \left(\log _2 N\right)^{ (m+2)(d-1)}.
\end{aligned}
\end{equation*}
For $p =\infty$, using Lemma \ref{lemma:philialpha_Lpbound} and Lemma \ref{lemma:sum_l}, we have
\begin{equation*}
\begin{aligned}
\|f-I_nf\|_{L_{\infty}(\Omega)} &\leq \sum_{|\bm{l}|_1 > n+d-1} \max _{\bm{i} \in \mathcal{I}_{\bm l}}\left|v_{\bm l, \bm i}\right|\cdot 1.117^d \\
&\leq \sum_{|\bm{l}|_1 > n+d-1}c(\bm{\alpha}) \cdot 2^{-d-|\bm{l} \odot(\bm{\alpha}+\bm{1})|_1} \cdot\|D^{\bm{\alpha}+\bm{1}} f\|_{L_{\infty}(\Omega)}\\
&\leq  C_{m,d} \|D^{\bm{m}+\bm{1}} f\|_{L_{\infty}(\Omega)} N^{-m-1} \left(\log _2 N\right)^{(m+2)(d-1)}.
\end{aligned}
\end{equation*}  The proof is complete.  
\end{proof}
\subsection*{Proof of Lemma \ref{lemma:approxtimes}.}
\begin{proof}
By scaling $x\to\frac{x}{M}, y\to\frac{y}{M}$, it suffices to prove the case $M=1$. Recall that
$$R_U\left(\frac{i}{2^U}\right)=\left(\frac{i}{2^U}\right)^2,\quad i=0,1,2, \ldots, 2^U,$$
and $R_U$ is a linear polynomial in each $[\frac{i}{2^U},\frac{i+1}{2^U}]$:
\begin{equation}\label{Eq:h_U}
R_U(x) = \frac{2i+1}{2^U}\left(x-\frac{i}{2^U}\right)+\left(\frac{i}{2^U}\right)^2,\quad x\in\left[\frac{i}{2^U},\frac{i+1}{2^U}\right].
\end{equation}
For $0\leq x\leq y\leq1$, we only need to consider the following cases:
\begin{itemize}
\item[(a)] $x \in [\frac{i}{2^U},\frac{i+1}{2^U}]$, $y \in [\frac{i+2n}{2^U},\frac{i+2n+1}{2^U}]$, $\frac{x+y}{2} \in [\frac{i+n}{2^U},\frac{i+n+1}{2^U}]$;
\item[(b)] $x \in [\frac{i}{2^U},\frac{i+1}{2^U}]$, $y \in [\frac{i+2n+1}{2^U},\frac{i+2n+2}{2^U}]$, $\frac{x+y}{2} \in [\frac{i+n+1/2}{2^U},\frac{i+n+1}{2^U}]$;
\item[(c)] $x \in [\frac{i}{2^U},\frac{i+1}{2^U}]$, $y \in [\frac{i+2n+1}{2^U},\frac{i+2n+2}{2^U}]$, $\frac{x+y}{2} \in [\frac{i+n+1}{2^U},\frac{i+n+3/2}{2^U}]$;
\end{itemize}
where  $n\geq0$ is some integer. In the rest of the proof, we only consider case (a) because the analysis for cases (b) and (c) is similar. Direct calculation implies
\begin{equation}\label{eqn:RUpositive}
\begin{aligned}
&R_U\left(\frac{x+y}{2}\right)-\frac{1}{4}R_U(x)-\frac{1}{4}R_U(y) \\
&=\frac{x}{2^U}\left(\frac{i}{2}+n+\frac{1}{4}\right)+\frac{y}{2^U}\left(\frac{i}{2}+\frac{1}{4}\right)-\frac{i^2+2ni+i+n}{2^{2U+1}}\\
&\geq\frac{4in+2i^2}{2^{2U+2}}\geq 0.
\end{aligned}
\end{equation}
In addition, when case (a) happens, we have $0\leq i\leq 2^U-1$,  $i+2n+1\leq 2^U$ and
\begin{equation}\label{eqn:RUhalfbound}
\begin{aligned}
&R_U\left(\frac{x+y}{2}\right)-\frac{1}{4}R_U(x)-\frac{1}{4}R_U(y) \\
&=\frac{x}{2^U}\left(\frac{i}{2}+n+\frac{1}{4}\right)+\frac{y}{2^U}\left(\frac{i}{2}+\frac{1}{4}\right)-\frac{i^2+2ni+i+n}{2^{2U+1}}\\
&\leq \frac{i+1}{2^{U+1}} \leq \frac{1}{2}. 
\end{aligned}
\end{equation}
Combining \eqref{Eq:h_U} with \eqref{eqn:RUpositive} and \eqref{eqn:RUhalfbound} verifies
$$
0\leq \widetilde{\times}_1(x,y)=2\left[R_U\left(\frac{x+y}{2}\right)-\frac{1}{4}R_U\left(x\right)-\frac{1}{4}R_U\left(y\right)\right] \leq 1.
$$
The accuracy of $\widetilde{\times}_{1}(x,y)$ is confirmed using \eqref{eqn:R_Uerror}:
\begin{equation*}
    \begin{aligned}
&|\widetilde{\times}_{1}(x,y)-xy|\\
&=4\left|R_U\left(\frac{x+y}{2}\right)-\left(\frac{x+y}{2}\right)^2-\frac{1}{4}R_U\left(x\right)+\frac{1}{4}x^2-\frac{1}{4}R_U\left(y\right)+\frac{1}{4}y^2\right|\\
&\leq\frac{1}{2^{2U}}.
\end{aligned}
\end{equation*}
\end{proof}

\subsection*{Proof of Lemma \ref{lemma:squareapproximation}}
\begin{proof}
For $\bm y \in \mathbb R^L$ and $n \geq 1$, let 
\begin{align*}
    S_n(\bm y)&= \frac{T_n(\bm y)}{4^n}, \quad S_0(\bm y) = \bm y,\\ 
    C_n(\bm y)&= \sum_{i=1}^n S_i(\bm y),\quad C_0 (\bm y) = \bm 0_{L}.
\end{align*}
Then 
$
R_U(\bm y) = \bm y-C_U(\bm y).
$
Representing $T_1$ by ReLU functions:
$$
T_1(\bm y )= 2\sigma(\bm y)-4 \sigma\left(\bm y-\bm {\frac{1}{2}}\right)+2 \sigma(\bm y-\bm 1),
$$ we could write $S_{n+1}$ as
$$
\begin{aligned}
S_{n+1}=\frac{T_{n+1}}{4^{n+1}} &= \frac{1}{4^{n+1}}\left[\sigma(2 T_n)-\sigma(4 T_n - 2) +\sigma(2 T_n - 2)\right]\\
&=\frac{1}{4^{n+1}}\left[\sigma(2\cdot 4^n S_n)-\sigma(4 \cdot4^n S_n - 2) +\sigma(2\cdot 4^n S_n - 2)\right]\\
&=\sigma\left(\frac{S_n}{2} \right)-\sigma\left(S_n - \frac{1}{2^{2n+1}}\right) +\sigma\left(\frac{S_n}{2} - \frac{1}{2^{2n+1}}\right).
\end{aligned}
$$
Using Lemma \ref{Le:Represent_shallow_net} with
$
\bm w =[\bm 0_{3L};1;\bm 0_{4L-1};1]
$ with filter length $7L$ and $\bm b=\bm0$, we find some $h_{L_0}$ with $L_0 \leq\left\lceil\frac{7L}{s-1}\right\rceil$ such that
$$
\bm y\stackrel{h_{L_0}}\longrightarrow\left[\bm 0_{3L}; \bm y; \bm{0}_{3 L}; \bm y\right].
$$
Suppose a $h_{L_n}$ with depth $L_n$ satisfying
$$
\bm y\stackrel{h_{L_n}}\longrightarrow\left[C_n(\bm y); \bm{0}_{2 L}; 2^nS_n(\bm y); \bm{0}_{3 L};\bm y\right]
$$
is available. We need to build another $h_{L_{n+1}}$ upon $h_{L_n}$ such that 
$$
\bm y\stackrel{h_{L_{n+1}}}\longrightarrow\left[C_{n+1}(\bm y); \bm{0}_{2 L}; 2^{n+1}S_{n+1}(\bm y);\bm{0}_{3 L};\bm y\right].
$$
To this end, we  combine the CNN from Lemma \ref{Le:Represent_shallow_net} with $\bm w=[1;\bm 0_{L-1};2;\bm 0_{L-1};1]$, 
$$
\bm b=[\bm 0_{L};-\bm 2_{L};-\bm 1_{L};\bm 0_{L}; -(\bm 2^{-n})_{L}; -(\bm 2^{-n})_{L};\bm 0_{2L};\bm -\bm2_{L};\bm -\bm1_{L}],
$$
(where $(\bm 2^{-n})_{L}$ denotes the vector $\left(2^{-n}, \ldots, 2^{-n}\right)^\top\in \mathbb R^L$) with $h_{L_n}$ to construct $h_{L_{n+1,1}}$ with depth $L_{n+1,1} \leq L_n +\left\lceil\frac{2L}{s-1}\right\rceil$ such that
$$
\bm y\stackrel{h_{L_{n+1,1}}}\longrightarrow\left[C_n(\bm y); \bm{0}_{2 L}; \sigma(2^nS_n(\bm y));\sigma(2^{n+1}S_n(\bm y)-\bm 2^{-n}); \sigma(2^nS_n(\bm y)-\bm 2^{-n});\bm{0}_{ L};\bm y\right].
$$
Using Lemma \ref{Le:Represent_shallow_net} again with $\bm w=[1;\bm 0_{L-1};-1;\bm 0_{L-1};1]$, $\bm b=[\bm 0_{2L}; -\bm 2_{3L};\bm 0_{L}; -\bm 2_{2L};\bm 0_{2L}]$ and composing the resulting CNN with $h_{L_{n+1,1}}$, we obtain $h_{L_{n+2,1}}$ with depth $L_{n+1,2} \leq L_{n+1,1} +\left\lceil\frac{2L}{s-1}\right\rceil$ such that
$$
\bm y\stackrel{h_{L_{n+1,2}}}\longrightarrow\left[C_n(\bm y); \bm{0}_{2 L}; 2^{n+1}S_{n+1}(\bm y);\bm{0}_{3 L};\bm y\right].
$$
Finally, $h_{L_{n+1}}$ follows from composing $h_{L_{n+1,2}}$ with the CNN by Lemma \ref{Le:Represent_shallow_net} with $\bm{w}=[2^{-n-1};\bm 0_{3L-1};1]$, $\bm b=[-\bm 1_{3L}; \bm 0_{4L};-\bm 1_{3L};\bm 0_{L}]$ and using $C_{n+1}(\bm y)=C_{n}(\bm y)+S_{n+1}(\bm y)$. The depth $L_{n+1} \leq L_{n+1,2} +\left\lceil\frac{3L}{s-1}\right\rceil$.

By repeating this process $U$ times,  we  obtain $h_{L_U}$ such that
$$
\bm y\stackrel{h_{L_U}}\longrightarrow\left[C_U(\bm y); \bm{0}_{2 L}; 2^U S_U(\bm y); \bm{0}_{3 L};\bm y\right].
$$
Composing the CNN from Lemma \ref{Le:Represent_shallow_net} with $\bm w=\left[1;\bm{0}_{7 L-1}; -1;\bm{0}_{L-1}; 1\right]$, $\bm b=[-\bm 2_{7L}; \bm 0_{L};-\bm 2_{L};\bm 0_{7L}]$ and $h_{L_U}$, we obtain $h_J$ with  $J \leq L_U +\left\lceil\frac{8L}{s-1}\right\rceil$ such that
$$
\bm y\stackrel{h_{J}}\longrightarrow\left[\bm y -C_U(\bm y);\bm{0}_{7 L};\bm y\right]=\left[R_U(\bm y); \bm{0}_{7 L}; \bm y\right].
$$
Recall $L_0 \leq\left\lceil\frac{7L}{s-1}\right\rceil$, $J - L_U \leq \left\lceil\frac{8L}{s-1}\right\rceil$ and 
$
L_{n+1}- L_n \leq \frac{7L}{s-1}+3.
$ As a result,
\begin{equation*}
    J = L_0+\sum_{n=0}^{U-1}(L_{n+1}-L_{n})+J-L_U \leq \frac{(7 U+15) L}{s-1}+3 U+2.  
\end{equation*}
The proof is complete.
\end{proof}
\subsection*{Acknowledgements}Y.~L. would like to thank Dr.~Tong Mao for helpful discussion about the non-negativity of $\widetilde{\times}_{M,U}$. The authors would like to thank the anonymous referee for
helpful suggestions  that improve the quality and presentation of this paper.
%\section*{Declarations}
%\subsection*{Funding}
This work is supported by the National Natural Science Foundation of China (no.~12471346) and the Fundamental Research Funds for the Zhejiang Provincial Universities (no. 226-2023-00039). 

%\subsection*{Conflict of Interest}
%The authors declare no competing interests.

\bibliographystyle{plainnat}
\bibliography{ref2,totalref}
\end{document}